\documentclass[letterpaper, 10 pt, journal, twoside]{IEEEtran}              
\usepackage{amssymb}
\usepackage{bbding}
\usepackage{xcolor}
\usepackage{amsmath} 
\usepackage{amssymb}
\usepackage{amsfonts}

\usepackage{amsthm}
\usepackage{graphicx,url} 
\usepackage{bm}
\usepackage{float}
\usepackage{makecell}
\let\labelindent\relax
\usepackage{enumitem}
\usepackage{subcaption}
\usepackage{wrapfig}
\usepackage{multirow}
\usepackage{diagbox}

\usepackage{cite}
\usepackage{soul}  
\usepackage[bookmarks=true]{hyperref}
\hypersetup{
    colorlinks=true,
    pdfborderstyle={/S/U/W 1},
    linkcolor=blue,
    filecolor=magenta,      
    urlcolor= blue,
    linkbordercolor=red,
}
\newcommand\rurl[1]{%
  \href{http://#1}{\nolinkurl{#1}}%
}

\usepackage{comment}

\def\black#1{{\color{black}#1}}

\let\abs\relax
\newcommand{\abs}[1]{\left\lvert#1\right\rvert}
\newcommand{\norm}[1]{\left\lVert#1\right\rVert}

\newcommand{\infnorm}[1]{\left\lVert#1\right\rVert_\infty}

\def\lone{{\mathcal{L}_1}}

\def\lonew{${\mathcal{L}_1}$ }

\def \loneAC {$\lone$AC}

\def\tilx{\tilde{x}}

\def\dotx{\dot x}
\def \hatx{\hat{x}}

\def \hsigma{\hat{\sigma}}

\def\mbR{\mathbb{R}}

\def\mbZ{\mathbb{Z}}

\def\mbZ{\mathbb{Z}}

\def\hsigma{\hat{\sigma}}

\def\mcX{\mathcal{X}}

\def\mcU{\mathcal{U}}

\def\trieq{\triangleq}

\newtheorem{lemma}{Lemma}

\theoremstyle{definition}  
\theoremstyle{definition} \newtheorem{assumption}{Assumption}
\theoremstyle{remark}  
\newtheorem{remark}{Remark}

\usepackage{accents}

\makeatletter
\def\cl@part {\@elt {chapter}}
\makeatother

\usepackage{cleveref}

\crefname{equation}{}{} 
\crefname{lemma}{Lemma}{Lemmas}
\crefname{theorem}{Theorem}{Theorems}
\crefname{table}{Table}{Tables}
\crefname{figure}{Fig.}{Figs.}
\crefname{remark}{Remark}{Remarks}
\crefname{assumption}{Assumption}{Assumptions}
\crefname{section}{Section}{Sections}
\crefname{definition}{Definition}{Definitions}
\crefname{algorithm}{Algorithm}{Algorithms}

\makeatletter
\renewcommand*\env@matrix[1][\arraystretch]{%
  \edef\arraystretch{#1}%
  \hskip -\arraycolsep
  \let\@ifnextchar\new@ifnextchar
  \array{*\c@MaxMatrixCols c}}
\makeatother

\def\hd{{\hat d}}
\def\mbZ{\mathbb{Z}}

\def\mcx{{\mathcal{X}}}

\def\mcU{{\mathcal{U}}}

\def \bbracket#1{\bm{[}#1\bm{]}}

\def\uRL{u_\textup{RL}} 

\def\uLone{u_{\mathcal L_1}} 
\begin{document}

\title{Improving the Robustness of Reinforcement Learning Policies  with \lonew Adaptive Control}

\author{Yikun Cheng$^\star$$^{1}$, Pan Zhao$^\star$$^{1}$, Fanxin Wang$^{1}$, Daniel J.~Block$^{2}$, Naira Hovakimyan$^{1}$
\thanks{
This work is supported by AFOSR and NSF under the RI grant \#2133656, NRI grant \#1830639, and AI Institute Planning grant \#2020289. {\it ($^\star$Yikun Cheng and Pan Zhao contributed equally to this work.) (Corresponding author: Pan~Zhao.)}}
\thanks{$^{1}$Yikun~Cheng, Pan~Zhao, Fanxin~Wang and Naira~Hovakimyan are with the Mechanical Science and Engineering
Department, University of Illinois at Urbana-Champaign, IL 61801, USA. Email: \texttt{\{yikun2, panzhao2, fanxinw2, nhovakim\}@illinois.edu}.}
\thanks{$^{2}$Daniel~J.~Block is with the Electrical and Computer Engineering Department, University of Illinois at Urbana-Champaign, IL 61801, USA. Email: \texttt{d-block@illinois.edu}.}
} 
\maketitle

\maketitle

\begin{abstract}
A reinforcement learning (RL) control policy could fail in a new/perturbed environment that is different from the training environment, due to the presence of dynamic variations. 
For controlling systems with continuous state and action spaces, we propose an add-on approach to robustifying a pre-trained RL policy by augmenting it with an \lonew adaptive controller (\loneAC). Leveraging the capability of an \loneAC~for fast estimation and active compensation of  dynamic variations, the proposed approach can improve the robustness of an RL policy which is trained  either in a simulator or in the real world without consideration of a broad class of 
dynamic variations. Numerical and real-world  experiments  empirically demonstrate the efficacy of the proposed approach in robustifying RL policies trained using both model-free and model-based methods. 
\end{abstract}


\section{Introduction}\label{sec:intro}
Reinforcement learning (RL) is a promising way to solve sequential decision-making problems \cite{sutton2018RL-intro}. In the recent years, RL has shown impressive or superhuman performance in 
control of complex robotic systems \cite{kaufmann2018droneracing-rl,hwangbo2019agile-legged}. 
An RL policy is often trained in a simulator and deployed in the real world. However, the discrepancy between the simulated and the real environment, known as the sim-to-real (S2R) gap, often causes the RL policy to fail in the real world. An RL policy may also be directly trained in a real-world environment; however,  the environment perturbation resulting from parameter variations, actuator failures and external disturbances can still cause the well-trained policy to fail. Take a delivery drone for example (Fig.~\ref{fig:adaptive-RL}). We could train an RL policy to control the drone in a nominal environment (e.g.,  nominal load, mild wind disturbances, healthy propellers, etc.); however, this policy could fail and lead to a crash when the drone operates in a new environment (e.g., heavier loads, stronger wind disturbances, loss of propeller efficiency, etc.). To a certain extent, the S2R gap issue can be considered as a special case of environment perturbation by treating the simulated and real environments as the old/nominal and new/perturbed environments, respectively.

\begin{figure}
    \centering
    \includegraphics[width=0.9\columnwidth]{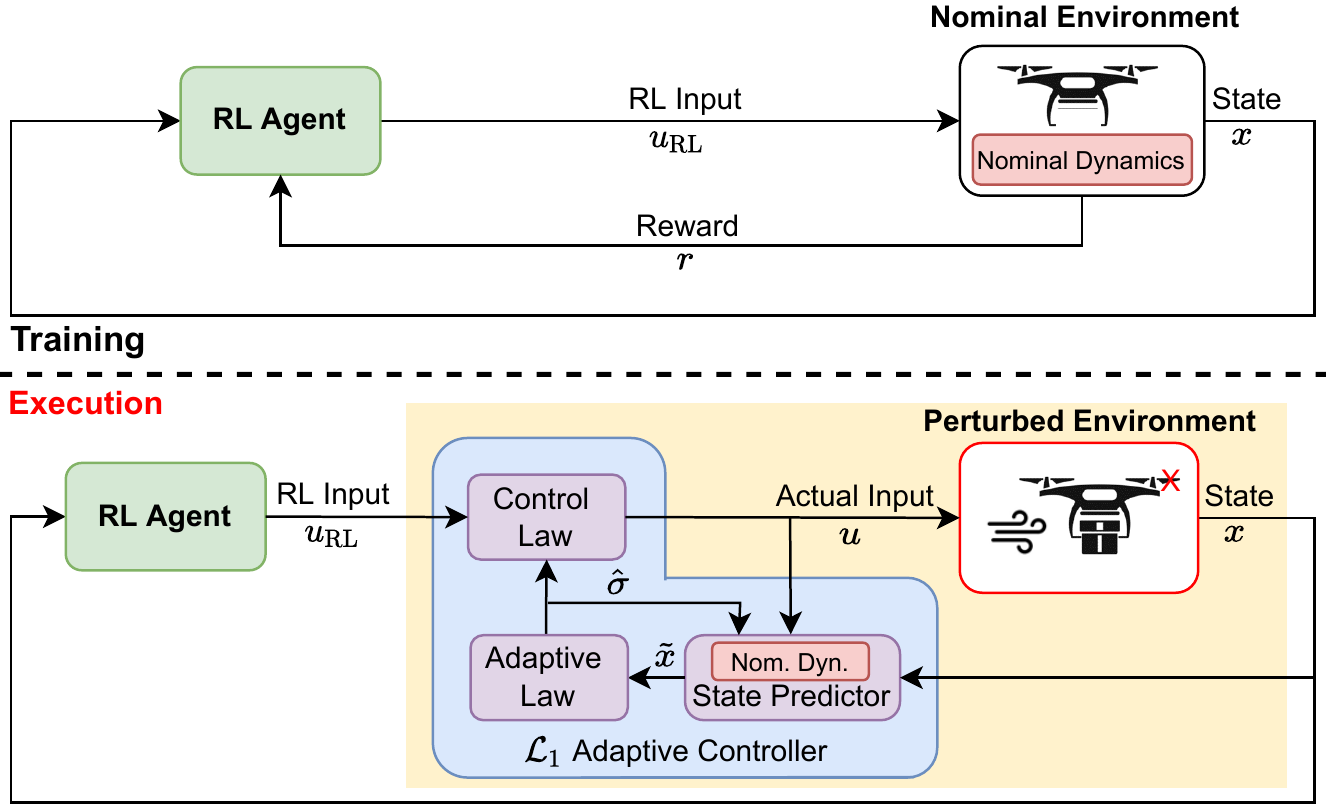}
    \caption{{Proposed approach to policy robustness improvement based on $\lone$ adaptive augmentation}}
    \label{fig:adaptive-RL}
    \vspace{-6mm}
\end{figure}
\subsection{Related work}
\noindent\textbf{Robust/adversarial training}: Domain/dynamics randomization was proposed to close the sim-to-real (S2R) gap \cite{tobin2017domainrand,peng2018sim2real,loquercio2019deep-drone-RL} when transferring a policy from a simulator to the real world. Robust adversarial training addresses the S2R gap and environment perturbations by formulating a two-player zero-sum game between the agent and the disturbance \cite{pinto2017robust-adversarial}. A similar idea was explored in \cite{abdullah2019wasserstein}, where Wasserstein distance was used to characterize the set of dynamics for which a robust policy was searched via solving a min-max problem. 

Though fairly general and applicable to a broad class of systems, these methods often involve tedious modifications to the training environment or the dynamics, which can only happen in a simulator. More importantly, the resulting {\it fixed} policies could overfit to the worst-case scenarios, and thus lead to conservative or degraded performance in other cases \cite{rice2020overfitting}.  

This issue is well studied in control community; more specifically, robust control \cite{Zhou98essentials} that  aims to provide performance guarantee for the worst-case scenario, often leads to conservative nominal performance.

\noindent{\black{\textbf{Post-training augmentation}}}:  Kim et al. \cite{kim2019rl-dob} proposed to use a disturbance observer (DOB) to improve the robustness of an RL policy, in which the mismatch between the simulated training environment and the testing environment is estimated as a disturbance and compensated for. A similar idea was pursued in \cite{guha2020mrac-rl}, which used a model reference adaptive control (MRAC) scheme to estimate and compensate for parametric uncertainties. Our objectives are similar to the ones in \cite{kim2019rl-dob} and \cite{guha2020mrac-rl}, but our approach and end results are different, as we address a broader class of dynamic uncertainties (e.g., unknown input gain that cannot be handled by \cite{kim2019rl-dob}, and time-dependent disturbances that cannot be handled by \cite{guha2020mrac-rl}), and we leverage the \lonew adaptive control architecture that is capable of providing guaranteed transient (instead of just asymptotic)  performance \cite{naira2010l1book}. Additionally, we validate our approach on real hardware, as opposed to merely in numerical simulations in   \cite{kim2019rl-dob,guha2020mrac-rl}. We note that \lonew adaptive control has been combined with model predictive control (MPC) with application to quadrotors \cite{ackerman2020l1-mppi}, and  it has been used for safe learning and motion planning applicable to a broad class of nonlinear systems  \cite{gahlawat2020l1gp,lakshmanan2020safe,gahlawat2021contraction-l1-gp}.

To put things into perspective, this paper is focused on applying the \lonew adaptive control architecture to robustify an RL policy. 
In terms of technical details, this paper considers more general scenarios, e.g., unmatched disturbances and unknown input gain, which were not considered in \cite{lakshmanan2020safe,gahlawat2021contraction-l1-gp}. 

\noindent \textbf{Learning to adapt}: Meta-RL has recently been proposed to achieve fast adaptation of a pre-trained policy in the presence of dynamic variations
\cite{finn2017MAML-meta,nagabandi2018adapt-metaRL,nagabandi2018deep-continual,saemundsson2018meta-rl-gp,xu2020task}.
Despite impressive performance mainly in terms of fast adaptation demonstrated by these methods, the intermediate policies learned during the adaptation phase will most likely still fail. This is because a certain amount of information-rich data needs to be collected in order to {\it learn} a good model and/or policy. On the other hand, rooted in  the theory of adaptive control and disturbance estimation,  \cite{ioannou2012robust,naira2010l1book,chen2015dobc}, our proposed method can {\it quickly  estimate} the discrepancy between a nominal model 
and the actual dynamics, and {\it actively compensate} for it in a timely manner. We envision that our proposed method can be combined with these methods to achieve robust and fast adaptation.

\subsection{Statement of contributions}
For controlling systems with continuous state and action spaces, we propose an {\it add-on} approach to robustifying an RL policy, which can be trained in standard ways without consideration of a broad class of potential dynamic variations. The essence of the proposed approach lies in augmenting it with  an \lonew adaptive control (\loneAC) scheme \cite{naira2010l1book} that quickly estimates and compensates for the uncertainties so that {\it the dynamics of the system in  the perturbed environment are close to that in the nominal environment}, in which the RL policy is trained and thus expected to function well. 
The idea  is illustrated in  Fig.~\ref{fig:adaptive-RL}. 

Different from most of existing robust RL methods using domain randomization or robust/adversarial training  \cite{tobin2017domainrand,peng2018sim2real,loquercio2019deep-drone-RL,pinto2017robust-adversarial,abdullah2019wasserstein}, the proposed approach can be used to robustify an RL policy, which is trained  either in a simulator or in the real world, using both model-free and model-based methods, without consideration of a broad class of uncertainties in the training. We empirically validate the approach on both numerical examples and real hardware.

\section{Problem Setting}

We assume that we have access to the system dynamics in the {\it nominal} environment, either simulated or in the real world, and they are described by a nonlinear control-affine model:
\begin{equation}\label{eq:dynamics-nom}
   \dot x(t) = f(x(t)) + g(x(t))u(t) \triangleq F_\textup{nom}(x(t),u(t)),
\end{equation}
where  $x(t)\in \mcx\subset {\mathbb{R}^n}$ and $u(t)\in \mcU\subset{\mathbb{R}^m}$ are the state and input vectors, respectively,  $\mcX$ and $\mcU$ are compact sets, \black{$f:{\mathbb{R}^{n}} \rightarrow {\mathbb{R}^{n}} $ and $g: {\mathbb{R}^{m}} \rightarrow {\mathbb{R}^{n\times m}}$}  are known and locally Lipschitz-continuous functions. Moreover, $g(x)$ has full column rank for any $x\in\mcX$. 
\begin{remark}
Control-affine models are commonly used for control design and can represent a broad class of mechanical and robotic systems. In addition, a control non-affine model can be converted into a control-affine model by introducing extra state variables (see e.g., \cite{takano2020robustcbf}). Therefore, the control-affine assumption is not very restrictive.
\end{remark}
The nominal model \eqref{eq:dynamics-nom} can be  from physics-based modeling, data-driven modeling or a combination of both. Methods exist for maintaining the control affine structure in data-driven modeling (see e.g.,   \cite{khojasteh2020probabilistic}). 

\begin{assumption}\label{assump:nominal-policy}
We have access to a {\it nominal} control policy, $\pi_{o}(x)$, which is trained using the {\it nominal} dynamics \eqref{eq:dynamics-nom} and thus functions well under such dynamics. Moreover, $\pi_0(x)$ is Lipschitz continuous in $\mcX$ with a Lipschitz constant $l_\pi$. 
\end{assumption}

The policy $\pi_{o}(x)$ can be trained either in a simulator or in the real world in the standard (i.e., non-robust) way, using either model-based and model-free methods. The Lipschitz continuity assumption is needed to derive an error bound for estimating the disturbances in \cref{sec:sub-analysis-l1-augmentation}. 
The nominal policy $\pi_0$ could fail in the perturbed environment due to the dynamic variations. We, therefore, propose a method to improve the robustness of this nominal policy in the presence of such dynamic variations, by leveraging \loneAC~\cite{naira2010l1book}. To achieve this, we further assume that the dynamics of the agent in the {\it perturbed} environment can be represented by 
\begin{equation}\label{eq:dynamics_perturb}
  \dot x = f(x) + g(x)\Lambda u + d(t,x),  
\end{equation}
where $\Lambda$ is an unknown input gain matrix, which satisfies \cref{assump:unknown-input-gain}, $d(t,x)$ is an unknown function that can capture parameter perturbations, unmodeled dynamics and external disturbances. \black{It is obvious that the perturbed dynamics \eqref{eq:dynamics_perturb} can be equivalently written as 
\begin{equation}\label{eq:dynamics-perturb-wrt-nom}
    \dot x= F_\textup{nom}(x,u)+\sigma(t,x,u),
    \vspace{-2mm}
\end{equation}}
where 
\begin{equation}\label{eq:sigma-txu-defn}
    \sigma(t,x,u) \trieq g(x)(\Lambda-I)u(t)+d(t,x).
\end{equation}

\begin{remark}
Uncertain input gain is very common in real-world systems. For instance,  actuator failures,  and variations in mass or inertia for force- or torque-controlled robotic systems,  normally induce such input gain uncertainty. For a single-input system, $\Lambda =0.6$ indicates a 40\% loss of the control effectiveness. Our representation of such uncertainty in \cref{eq:dynamics_perturb} is broad enough to capture a large class of scenarios, while still allowing for effective compensation of such input gain uncertainty using \loneAC~(detailed in Section~\ref{sec:adaptive-rl}). 
\end{remark}
To provide a rigorous treatment, we make the following assumptions on the perturbed dynamics \cref{eq:dynamics_perturb}. 
\begin{assumption}\label{assump:unknown-input-gain} 
The matrix $\Lambda$ in  \cref{eq:dynamics_perturb} is an unknown strictly row-diagonally dominant matrix with $\textup{sgn}(\Lambda_{ii})$ known. Furthermore, there exists a compact convex set $\mathbb{\Lambda}$ such that  $\Lambda\in \mathbb{\Lambda}$. 
\end{assumption}
\begin{remark}
The first statement in \cref{assump:unknown-input-gain} indicates that $\Lambda$ is always non-singular with known sign for the diagonal elements, and is often needed in applying adaptive control methods to mitigate the effect of uncertain input gain (see \cite[Sections 6 and 7]{ioannou2012robust}). Without loss of generality, we further assume that $\mathbb{\Lambda}$ in \cref{assump:unknown-input-gain} contains the $m$ by $m$ identity matrix, $I$.
\end{remark}

\begin{assumption}\label{assump:lipschitz-bound-fg}
There exist positive constants $l_d,~l_d^\prime,$ $b_d,$ $l_{f},$ and $l_{g}$  such that for any $x,y \in \mcX$ and $t,\tau\geq 0$, the following inequalities hold:
\begin{align} 
\left\| {d(t,x) - d(\tau ,y)} \right\| &\le {l_d}\left\| {x - y} \right\| + {l_d^\prime}\abs{t-\tau}, \label{eq:d-lipschitz-cond}\\
\left\| {d(t,0)} \right\| & \le {b_d},  \label{eq:d-x0-bound}\\
\left\| {f(x) - f(y)}\right\| & \le l_{f}\left\| {x - y}\right\|, \label{eq:f-lipschitz-cond}
\\
\left\| {g(x) - g(y)}\right\| & \le l_{g}\left\| {x - y}\right\|. \label{eq:g-lipschitz-cond}
\end{align}
\end{assumption}
\begin{remark}
This assumption essentially indicates that the rate of variation of $d(t,x)$ with respect to both $t$ and $x$, and of $f(x)$ and $g(x)$ with respect to $x$,  in $\mcX$, are bounded. It is needed for deriving the  theoretical error bounds (in \cref{lemma:estiamte-error-bound}) for estimating the lumped disturbance, $\sigma(t,x,u)$. 
\end{remark}
The problem we are tackling can be stated as follows. 
\textbf{Problem Statement}: 
Given an RL policy $\pi_{o}(x)$ well trained in a nominal environment with the nominal dynamics \cref{eq:dynamics-nom}, assuming the dynamics in the perturbed environment are represented by \cref{eq:dynamics_perturb} satisfying \cref{assump:unknown-input-gain,assump:lipschitz-bound-fg}, provide a solution to improve the robustness of the policy $\pi_{o}(x)$ in the perturbed environment.

\section{ \lonew Adaptive Augmentation for RL Policy Robustification}\label{sec:adaptive-rl}
\subsection{Overview of the proposed approach}
The idea of our proposed approach is depicted in Fig.~\ref{fig:adaptive-RL}. With our approach, the {training} phase is standard: the nominal policy can be trained using almost any RL methods (both model-free and model-based) in a nominal environment. After getting a nominal policy that functions well in the nominal environment, 
for {\it policy  execution}, an \lonew controller is designed to augment and work together with the nominal policy. The \lonew controller uses the dynamics of the nominal environment \eqref{eq:dynamics-nom} as an internal nominal model, estimates the discrepancy between the nominal model and the actual dynamics and compensates for this discrepancy so that {\it the actual dynamics with the \lonew controller} (illustrated by the shaded area of Fig.~\ref{fig:adaptive-RL}) {\it are close to the nominal dynamics}. Since the RL policy is well trained using the nominal dynamics, it is expected to function well in the presence of the dynamic variations {\it and} the \lonew augmentation. 

\subsection{RL training for the nominal policy}
As mentioned before, the policy can be trained in the standard way, using almost any RL method including both model-free and model-based one. The only requirement is that one has access to the nominal dynamics of the training environment in the form of \eqref{eq:dynamics-nom}. 


As an illustration of the idea, for the experiments in Section~\ref{sec:experiments},  we choose PILCO \cite{deisenroth2011pilco}, a model-based policy search method using Gaussian processes, 
soft actor-critic \cite{haarnoja2018soft-sac}, a state-of-the-art model-free deep RL method, and
a trajectory optimization method based on differential dynamic programming (DDP) \cite{tassa2012synthesis}
to obtain the nominal policy. 


\subsection{\lonew adaptive augmentation for policy robustification}\label{sec:sub-l1-augmentation}
In this section, we explain how an \loneAC~scheme can be designed to augment and robustify a nominal RL policy. An \lonew controller mainly consists of three components: a state predictor, an adaptive law,  and a low-pass filtered control law. The state predictor is used to predict the system's state evolution, and the prediction error is subsequently used in the adaptive law to update the disturbance estimates. The control law  aims to compensate for the estimated disturbance.
For the perturbed system \eqref{eq:dynamics_perturb} with the nominal dynamics \eqref{eq:dynamics-nom}, the {\bf state predictor} is given by
\begin{equation}\label{eq:state-predictor}
    \begin{aligned}
        \dot{\hat{x}}(t)=F_{\textup{nom}}(x,u)+\hsigma(t) -a \tilde{x}(t), 
    \end{aligned}
\end{equation}
where $\tilx (t)\triangleq \hatx(t) - x(t)$ is the prediction error,  $a$ is a positive constant, $\hsigma(t)$ is the estimation of the lumped disturbance, $\sigma(t,x,u)$, at time $t$. 
Following the piecewise-constant (PWC) {\bf adaptive law} (which connects with the CPU sampling time) \cite[Section~3.3]{naira2010l1book}, the disturbance estimates are updated as
\begin{equation}\label{eq:adaptive_law}
\begin{split}
  \hsigma(t)
    &=  \hsigma(iT)
    , \quad t\in [iT, (i+1)T),\\
\hsigma(iT) &= - \frac{a}{{{e}^{aT}}-1}\tilde{x}(iT),
\end{split}
\end{equation}
for $i=0,1,\cdots$, where $T$ is the estimation sampling time. With $\hsigma(t)$, we further compute 
\begin{equation}
    \begin{bmatrix}
   \hsigma_m(t) \\
   \hsigma_{um}(t) 
\end{bmatrix} = \left[g(x)\   g^{\perp}(x)\right]^{-1}\hsigma(t), 
\end{equation}
where $\hsigma_m(t)$ and $\hsigma_{um}(t)$ are the {\it matched} and {\it unmatched} disturbance estimates, respectively,  $g^\perp (x)\in\mbR^{n-m}$ satisfies $g(x)^\top g^\perp (x) = 0$, and $\textup{rank}\left(\left[g(x)\   g^{\perp}(x)\right]\right) = n$ for any $x\in\mcX$. 
\black{From \eqref{eq:dynamics-perturb-wrt-nom} and \eqref{eq:state-predictor}, we see that the total or lumped disturbance  $\sigma(t,x,u)$, is estimated by  
$\hat\sigma(t) \trieq g(x)\hsigma_m(t)+g^\perp(x){\hsigma_um}(t)$}. 
The {\bf control law} is given by 
\begin{equation}\label{eq:adaptive-control-law}
    \begin{split}
        u(t) &= u_\textup{RL}(t) + u_{\mathcal{L}_1}(t), \\
        u_{\mathcal{L}_1}(s) &=  -C(s)\mathfrak{L}[\hsigma_m(t)],
    \end{split}
\end{equation}
where $u_\textup{RL}(t) = \pi_0(x(t))$ is the control command from the nominal RL policy, $u_{\mathcal{L}_1}(s)$ is the Laplace transform of the \lonew control command $u_{\mathcal{L}_1}(t)$, $\mathfrak{L}[\cdot]$ denotes the Laplace transform, and $C(s)\trieq K(sI+K)^{-1}$ is an $m$ by $m$ transfer matrix consisting of low-pass filters with $K\in \mbR^{m\times m}$. 
\begin{table*}[]
\centering
\caption{\black{Comparison with existing approaches to improving the robustness of RL policies}}\label{table:comparison-w-existing-methods}
\color{black}
\begin{tabular}{|c|c|ccc|}
\hline
\multirow{2}{*}{}                & \multirow{2}{*}{\makecell{\bf Robust/Adversarial \\ {\bf Training} \cite{tobin2017domainrand,peng2018sim2real,loquercio2019deep-drone-RL,pinto2017robust-adversarial,abdullah2019wasserstein}}} & \multicolumn{3}{c|}{\bf Post-Training Augmentation}                                                                                                                                                        \\ \cline{3-5} 
  &                                      & \multicolumn{1}{c|}{{ \bf MRAC} \cite{guha2020mrac-rl}}   & \multicolumn{1}{c|}{{\bf DOB} \cite{kim2019rl-dob}}                                               & {{\bf$\bm{\mathcal L_1}$AC}~(ours)}                                                                                           \\ \hline\hline
Complexity of training           & High                                         & \multicolumn{3}{c|}{Low}                                                                   \\ \hline
Training environment & Simulated & \multicolumn{3}{c|}{Simulated \& Real-world}    \\ \hline
Restrictions on structure of dynamics    & {Low}  & \multicolumn{3}{c|}{High (control-affine \& continuous)}  \\ \hline
Control inputs &  {Multiple} & \multicolumn{1}{c|}{Single} & \multicolumn{1}{c|}{Single} & {Multiple}\\ \hline
Restrictions on uncertainties    & {Low}                                             &  \multicolumn{1}{c|}{\begin{tabular}[c]{@{}c@{}}High (matched \\ parametric  uncertainties)\end{tabular}} & \multicolumn{1}{c|}{\begin{tabular}[c]{@{}c@{}}High (matched \\ disturbances)\end{tabular}} & \begin{tabular}[c]{@{}c@{}}Medium (matched \\ uncertainties and disturbances)\end{tabular} \\ \hline
Control policy after training  &  Fixed   & \multicolumn{3}{c|}{Adapted Online}                                                                                       \\ \hline
Validation &  {Sims \& Experiments} &  \multicolumn{1}{c|}{Sims}  &  \multicolumn{1}{c|}{Sims}  &  {Sims \& Experiments} \\\hline
\end{tabular}
\vspace{-2mm}
\end{table*}

\black{
\begin{remark}
As it can be seen from \cref{eq:adaptive-control-law,eq:adaptive_law,eq:state-predictor}, in an \loneAC~scheme with a PWC adaptive law \cite[section~3.3]{naira2010l1book}, all the dynamic uncertainties (such as parametric uncertainties, unmodeled dynamics and external disturbances) are {\it lumped} together and estimated as a total disturbance. This is different from most adaptive control schemes \cite{ioannou2012robust}, which rely on a parameterization of the uncertainty to design adaptive laws for updating parameter estimates and usually consider only stationary uncertainties that do not directly depend on  time.  
\end{remark}}
Details on deriving the estimation and control laws can be found in \cite{zhao2020aR-cbf,zhao2021RALPV-tac}. 
\black{The working principle of the \lonew controller can be summarized as follows: the state predictor \eqref{eq:state-predictor} and the adaptive law \cref{eq:adaptive_law} can accurately estimate the lumped disturbances, $\hsigma_m(t)$ and $\hsigma_{um}(t)$. In fact, under certain conditions, a bound on the estimation error, $\hsigma(t)-d(t,x)$,  can be derived and is included in 
Appendix A. Additionally, the control law \eqref{eq:adaptive-control-law} mitigates the effect of  disturbances by cancelling those within the bandwidth of the low-pass filter.} Note that unmatched disturbances (also known as mismatched disturbances in the disturbance-observer based control literature \cite{chen2015dobc}) cannot be directly canceled by control signals and are more challenging to deal with. 
\black{\begin{remark}
In designing the \lonew controller consisting of \cref{eq:adaptive-control-law,eq:adaptive_law,eq:state-predictor}, we assume that the states are measured without noise. In practice, as long as the estimation sampling time is not too small and the filter bandwidth is not too large, moderate measurement noise that always exists in real-world systems  usually does not cause big issues, as demonstrated by the hardware experiments in \cref{sec:sub-exp-pendubot-realworld}.
\end{remark}}
\begin{remark}
Variants of the proposed \loneAC~law \cref{eq:state-predictor,eq:adaptive_law,eq:adaptive-control-law} have been used to augment other baseline controllers (e.g., PID, linear quadratic regulator, MPC), as demonstrated in numerous applications and flight tests, \cite{naira2010l1book}. \end{remark}

{\subsection{Analysis of the \lonew adaptive augmentation}\label{sec:sub-analysis-l1-augmentation}
In this section, we provide an analysis of the \lonew  augmentation presented in \cref{sec:sub-l1-augmentation} and explain the cases under which its performance can be limited. The working principle of the \lonew controller can be summarized as follows: the state predictor \eqref{eq:state-predictor} and the adaptive law \cref{eq:adaptive_law} can accurately estimate the lumped disturbances, $\hsigma_m(t)$ and $\hsigma_{um}(t)$, while the control law \eqref{eq:adaptive-control-law} mitigates the effect of matched disturbance, $\hsigma_m(t)$, by cancelling it within the bandwidth of the low-pass filter.
\subsubsection{Estimation error bound} Next, we will show that under certain assumptions, an error bound in estimating the lumped disturbance $\sigma(t,x,u)$ (and thus the matched and unmatched components) can be derived. Furthermore, this bound can be arbitrarily reduced by decreasing the estimation sample time $T$, which indicates that the estimation after one estimation sampling interval can be arbitrarily accurate. 
\begin{lemma}\label{lemma:d-bound}
Given the perturbed dynamics \eqref{eq:dynamics_perturb} subject to \cref{assump:unknown-input-gain,assump:lipschitz-bound-fg},  if $x(t)\in \mcX$ and $u(t)\in \mcU$ for any $t$ in $[0,\tau]$, we have that 
\begin{equation}\label{eq:d-xdot-bound}
\begin{gathered}
        \norm{d(t,x)}\leq \theta, \;  \norm{g(x)(\Lambda-I) u} \leq \rho,\   \\ \norm{\sigma(t,x,u)}\leq \rho +\theta,\; \norm{\dot{x}(t)}\leq \phi, 
\end{gathered}
\end{equation}
for any $t$ in $[0,\tau]$, where
\begin{align}
    \theta &\trieq l_d \max_{x\in \mcX} \norm{x} + b_{d}, \label{eq:theta-defn}\\
    \rho & \trieq \max_{\Lambda\in \mathbb{\Lambda}}\norm{\Lambda-I}\max_{x\in \mcX}\norm{g(x)}\max_{u\in \mcU}\norm{u}, \\
    \phi &\trieq  \max_{x\in \mcX}\norm{f(x)}+\max_{x\in \mcX}\norm{g(x)}\max_{u\in \mcU}\norm{u} +\theta +\rho. \label{eq:phi-defn}
\end{align}
\end{lemma}
\begin{proof} See \cref{sec:sub-proof-lemma-d-bound}.
\end{proof}
Let us define: 
\begin{align}
      \hspace{-3mm}     l_u^\prime  \!\trieq & l_\pi \phi + \norm{K} \nonumber\\
    & \cdot \! (\max_{u\in \mcU}\norm{u}+\sqrt{n}{e^{-aT}}(\theta+\rho)\max_{x\in\mcX}\norm{g^{+}(x)}), \label{eq:l_u-defn}\\
     \hspace{-3mm}   \eta_{1}  \!\trieq & l_d^\prime+l_d\phi, \label{eq:eta1-defn}\\
  \hspace{-3mm}   \eta_{2}  \!\trieq & (l_{g}\phi \max_{u\in \mcU}\norm{u(t)} \!+\! l_{u} ^\prime\max_{x\in \mcX}\norm{g(x)})\max_{\Lambda\in \mathbb{\Lambda}}\norm{\Lambda\!-\!I}, \label{eq:eta2-defn}\\
     \hspace{-3mm}   \gamma(T)  \!\trieq& 2\sqrt{n}(\eta_{1}+\eta_{2}) T +\sqrt{n} (1- e^{-aT})(\theta+\rho), \label{eq:gammaTs-defn}
\end{align}
where $g^+(x)$ is the pseudoinverse of $g(x)$, and $\theta$ and $\phi$ are defined in \eqref{eq:theta-defn} and \eqref{eq:phi-defn}, respectively. We next establish the estimation error bounds associated with the estimation scheme in \eqref{eq:state-predictor} and \eqref{eq:adaptive_law}.
\begin{lemma}\label{lemma:estiamte-error-bound}
Given the perturbed dynamics \eqref{eq:dynamics_perturb} subject to \cref{assump:unknown-input-gain,assump:lipschitz-bound-fg}, and the estimation law in \eqref{eq:state-predictor} and \eqref{eq:adaptive_law}, if $x(t)\in \mcX$ and $u(t)\in \mcU$ for any $t$ in $[0,\tau]$ with $\tau>T$, the estimation error can be bounded as 
\begin{align}\label{eq:estimation_error_bound}
     &\norm{\sigma(t,x,u)-\hat{\sigma}(t)} 
     \leq
     \left\{
    \begin{array}{ll}
    \theta+\rho,\quad &\forall~ 0\leq t<T, \\
    \gamma(T), \quad &\forall~ T\leq t< \tau,
    \end{array}
    \right.
\end{align}
with $\sigma(t,x,u)$ defined in \cref{eq:sigma-txu-defn}. Moreover, 
$\lim_{T\rightarrow 0} \! \gamma(T) \!=\! 0.$
\end{lemma}
\begin{proof}
See \cref{sec:sub-proof-lemma-estimate-err-bnd}. 
\end{proof}
}
\begin{remark}
\cref{lemma:estiamte-error-bound} essentially states that under \cref{assump:unknown-input-gain,assump:lipschitz-bound-fg} and the assumed boundedness of $x$ and $u$, the error for estimation of the lumped disturbance is always bounded. Furthermore, the error after one sampling interval can be arbitrarily reduced (by decreasing $T$). In practice, the size of $T$ is limited by the computational hardware and measurement noise. 
\end{remark}

\subsubsection{Limitations of the proposed approach}
As mentioned before, the control law \cref{eq:adaptive-control-law} only tries to cancel the matched disturbance $\hsigma_m(t)$, while ignoring the unmatched disturbance $\hsigma_{um}(t)$. Dealing with unmatched disturbance in the nonlinear setting has been a long-standing challenging problem for adaptive or disturbance observer based control methods, and need other methods, e.g., those based on robust control \cite{zhao2021tube-rccm}. As a result, 
when the unmatched disturbance dominates the total disturbance, the performance of the proposed approach will be limited. This is demonstrated in \cref{sec:experiments}, e.g., in the quadrotor example in the presence of wind disturbances.

\black{
\subsection{Comparison with existing approaches}
The comparison of our proposed approach with existing approaches is summarized in \cref{table:comparison-w-existing-methods}. Our approach falls into the category of post-training augmentation (PTA), which does not require a special training process such as randomizing parameters and adding disturbances, and allows the training to be done in both simulated and real-world environments, as opposed to robust/adversarial training (RAT) methods. 
Additionally, RAT methods aim to find a fixed policy
for all possible realizations of uncertainties, which could be infeasible when the range of uncertainties is large. Compared to existing PTA methods based on MRAC and DOB, our approach is able to deal with a broader class of uncertainties, and is validated on real hardware.}

\black{On the other hand, similar to other PTA methods, our approach needs the dynamics to be continuous and have a control-affine form, and can only effectively compensate for the matched disturbance. Dealing with the unmatched disturbances in the nonlinear setting has been a long-standing challenging problem for adaptive or DOB-based control methods,  other methods, e.g., those based on robust control \cite{zhao2022tube-rccm-ral}, must be considered. As a result, 
when the unmatched disturbance dominates the total disturbance, the performance of the proposed approach will be limited. This is demonstrated in \cref{sec:experiments}, e.g., in the quadrotor example in the presence of wind disturbances.}
\section{Experiments}\label{sec:experiments}

\black{We now apply the proposed approach to three systems, namely a cart-pole, a Pendubot and a 3-D quadrotor. In particular, for the Pendubot, experiments on real hardware are also conducted}. An overview of the systems and test settings is given in \cref{table:test_system}. The dynamic models for these systems are included in \cref{sec:sub-dyn-model-systems-experiments}. 
{\setlength{\tabcolsep}{1pt}
\begin{table}[h]
\vspace{1mm}
\begin{center}
\caption{An overview of testing systems and settings}\label{table:test_system}
\hspace{-2.8mm}
\begin{tabular}{|c|c|c|c|c|c|}
\hline
{\bf System}    & {\bf \makecell{~State/input~ \\
Dimension}} & {\bf \makecell{~Policy Search~ \\ Methods}}                              & {\bf \makecell{Test \\ Environments}} &  {\bf \makecell{W/ Unmatched\\Disturbances~}} \\ \hline \hline
Cart-pole & 4/1                & PILCO                                               &                    Simulation (\ref{sec:sub-exp-cartpole})                    & Yes  \\ \hline
Pendubot  & 4/1                & \makecell{PILCO, SAC \\ \! \&  \!  DR-SAC} & \makecell{Simulation (\ref{sec:sub-exp-pendubot-sim})      \\ \& \!\!\! Hardware (\ref{sec:sub-exp-pendubot-realworld})    }                   & Yes   \\ \hline
Quadrotor & 12/4               & DDP                 & Simulation  (\ref{sec:sub-exp-quadrotor})                       & Yes   \\ \hline
\end{tabular}
\end{center}
\vspace{-3mm}
\end{table}}

\subsection{Cart-pole swing-up and balance in simulations}\label{sec:sub-exp-cartpole}
The system states include cart position $x_c$ and velocity $\dot x_c$, and pole angle $\theta$ and angular velocity $\dot \theta$. The input is the force applied to the cart. 
The nominal value of the key parameters in the dynamics are  $M=0.5\!~\mathrm{kg}$ (cart mass),  $m=0.5\!~\mathrm{kg}$  (pole mass),  $l_\textup{pole}=0.6$ m   (pole length). The pole is roughly hanging straight down ($\theta=0$) with small random perturbations at the beginning. The goal is to search for a policy that can swing up the pole and balance it at the straight up position (corresponding to $x_c=0$ and $\theta = 180^{\circ}$).

We used PILCO \cite{deisenroth2011pilco}~
to search for a policy in the nominal environment defined by the nominal values mentioned above. PILCO  adopts Gaussian processes (GPs) 
to learn the systems dynamics, uses the learned dynamics together with uncertainty propagation (e.g., based on moment matching or linearization) to predict the cost, and then applies gradient descent to search for the optimal policy. PILCO achieved unprecedented records in terms of  data-efficiency in RL.  

We next perturb the environment to test the robustness of the nominal policy with and without \lonew augmentation. For \lonew~augmentation design,  we use the physics-based model with the nominal parameter values as the nominal model, instead of the GP model learned during policy training, for simplicity. Moreover, the parameters in \cref{eq:state-predictor,eq:adaptive_law,eq:adaptive-control-law} were chosen to be $a=10$, $T=0.002$ second, and $K=200$, and {\it fixed} across all the tests. Figure~\ref{fig:change_single_param} shows the results in the presence of perturbations in the cart mass and pole length, while the perturbations in the latter induced unmatched disturbances. One can see that the \lonew augmentation significantly improves the robustness of the PILCO policy. For instance, PILCO plus \lonew augmentation was  able to consistently achieve the goal even when the cart mass was perturbed to $3\!~\mathrm{kg}$ (six times of its nominal value) or when the pole length was reduced to $0.2$ m (one third of its nominal value). 
\begin{figure}[h]
\vspace{-1mm}
\centering
\includegraphics[width=0.9\linewidth]{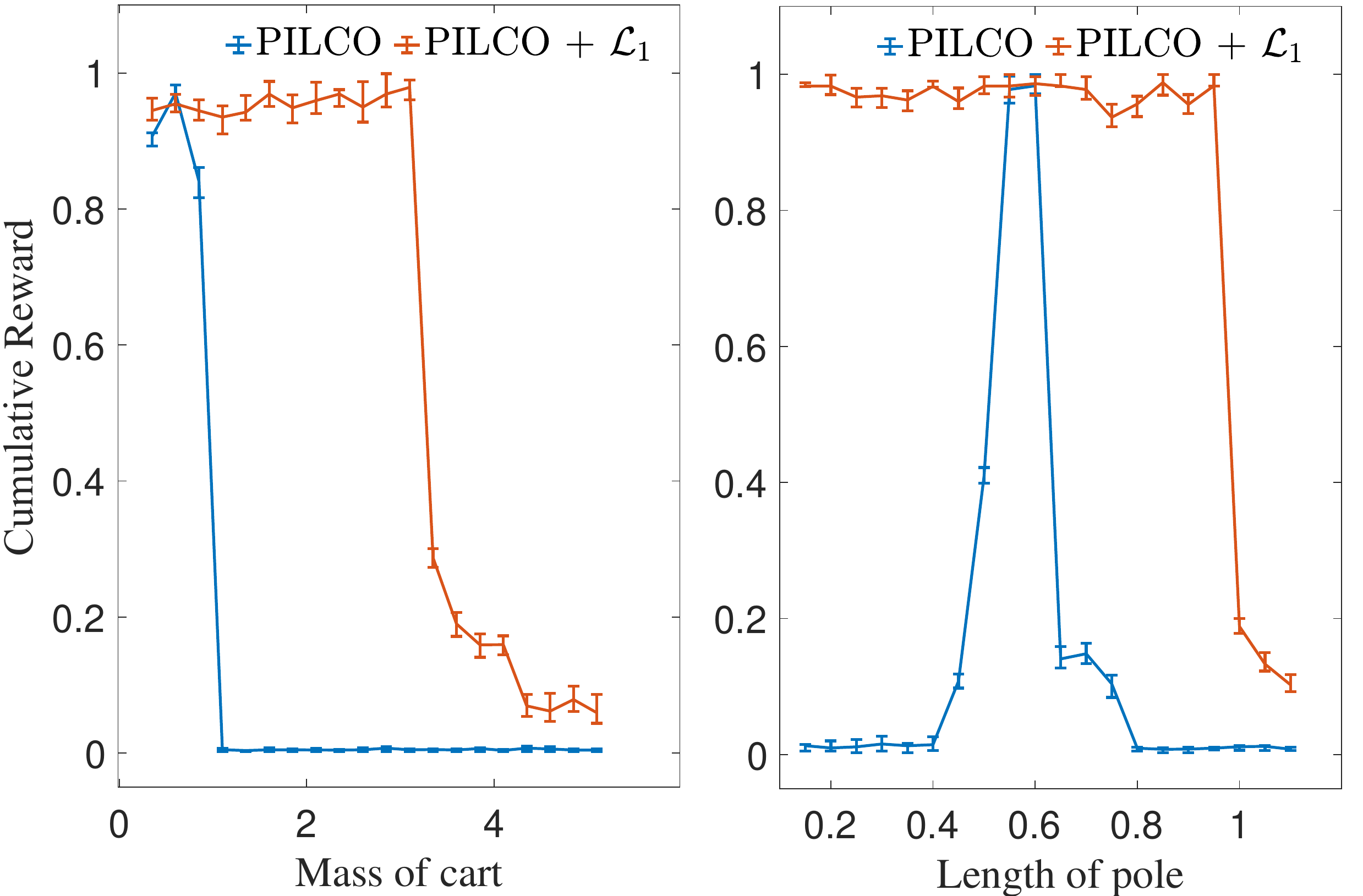}
  \vspace{1mm}
  \caption{Results in the presence of perturbations in cart mass and pole length. Ten trials were performed and average results with variances are shown for each perturbation case. Cumulative reward is normalized.}
  \label{fig:change_single_param}
    \vspace{-1mm}
\end{figure}

We further performed testing under ten scenarios, each of which involves random joint perturbations in the cart mass, pole mass and length parameters, in the range of $M\in[0.1,5]\!~\mathrm{kg},\  m\in[0.1,5]\!~\mathrm{kg},\ l_\textup{pole}\in[0.3,1]$ m. The sampled parameters and the success/failure results for each scenario are shown in Fig.~\ref{fig:10_samples}. Once again, the \lonew augmentation significantly improved the policy robustness, as validated by the much higher success rate. Also, it is not a surprise that PILCO plus \lonew augmentation failed under Scenarios 9 and 10 as these two scenarios involve significant perturbations in pole mass (and additionally in pole length for Scenario 9), which induces unmatched disturbances that could not be compensated for.
\begin{figure}[h]
\vspace{-1mm}
\centering
  \includegraphics[width=0.9\linewidth]{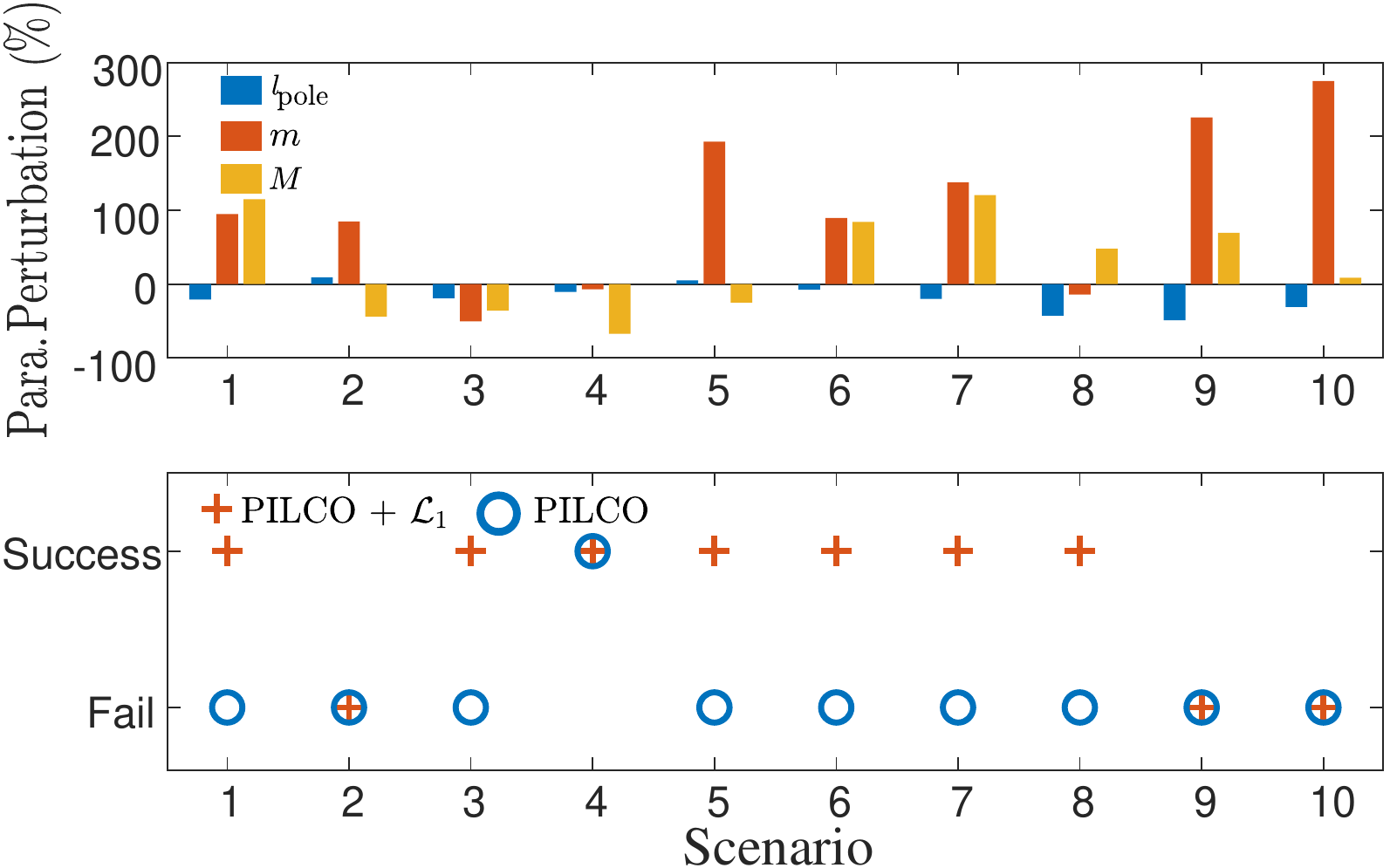}
  \caption{Results (bottom) under ten random perturbations in the cart mass, pole mass and length (percentage perturbation with respect to the nominal value shown at the top)}
  \label{fig:10_samples}
  \vspace{1mm}
\end{figure}




\subsection{Pendubot swing-up and balance in simulations}\label{sec:sub-exp-pendubot-sim}
As depicted in \cref{fig:pend_demo}, the Pendubot is a mechatronic system consisting of two rigid links interconnected by revolute joints with the second joint unactuated. 
The states of the system include the angles and angular rates of the two links, and the control input is the torque applied to Link 1. The task is to swing up the links from initial states $[q_{1}, q_{2}] = [\pi, \pi]$ to the right-up position $[q_{1}, q_{2}] = [0, 0]$ and balance them there, as illustrated in Fig.~\ref{fig:pend_demo}. \black{The same reward function is used for training SAC and DR-SAC policies and defined by 
\begin{equation}\label{pend_reward_function}
    r = \!- 3(\abs{\sin(q_{1})}\!+\! \abs{\cos(q_{1})-1}\!+\! \abs{\sin(q_{2})}\!+\! \abs{\cos(q_{2})\!-\!1}).
\end{equation}}
\setlength{\tabcolsep}{1pt}
\begin{table}[t]
\color{black}
\begin{center}
\caption{\black{Selected training settings for Pendubot}}\label{table:training_policy}
\hspace{-2.8mm}
\resizebox{0.5\textwidth}{!}{%
\begin{tabular}{|c|c|c|c|}
\hline
\bf{Setting} &  \bf{Parameters}  & \bf{\makecell{Input\\Limit}} & \bf{Policy}  \\ \hline\hline
I    & $\Lambda=1.0$, $m_1 = 0.12\!~\mathrm{kg}$, $m_2=0.11\!~\mathrm{kg}$  & $4\!~\mathrm{Nm}$ & SAC \\ \hline
II  & $\Lambda\!\in\! [0.3,1]$, $m_1 \!\in\! 0.12[1, 6]\!~\mathrm{kg}$, $m_2\!\in\!0.11[1,6]\!~\mathrm{kg}$ & $6\!~\mathrm{Nm}$ &DR-SAC1\\ \hline
III  & $\Lambda\!\in\! [0.3,1]$, $m_1 \!\in\! 0.12[1,6]\!~\mathrm{kg}$, $m_2\!\in\!0.11[1,6]\!~\mathrm{kg}$  & $9 \!~\mathrm{Nm}$ & DR-SAC2 \\ \hline
IV & $\Lambda\!\in\! [0.5,1]$, $m_1 \!\in\! 0.12[1,4]\!~\mathrm{kg}$, $m_2 \!\in\!0.11[1,4]\!~\mathrm{kg}$ & $6\!~\mathrm{Nm}$ & DR-SAC3 \\ \hline
\end{tabular}%
}
\vspace{-3mm}
\end{center}
\end{table}
\begin{figure}[h!]
    \centering
    \includegraphics[width=1\linewidth]{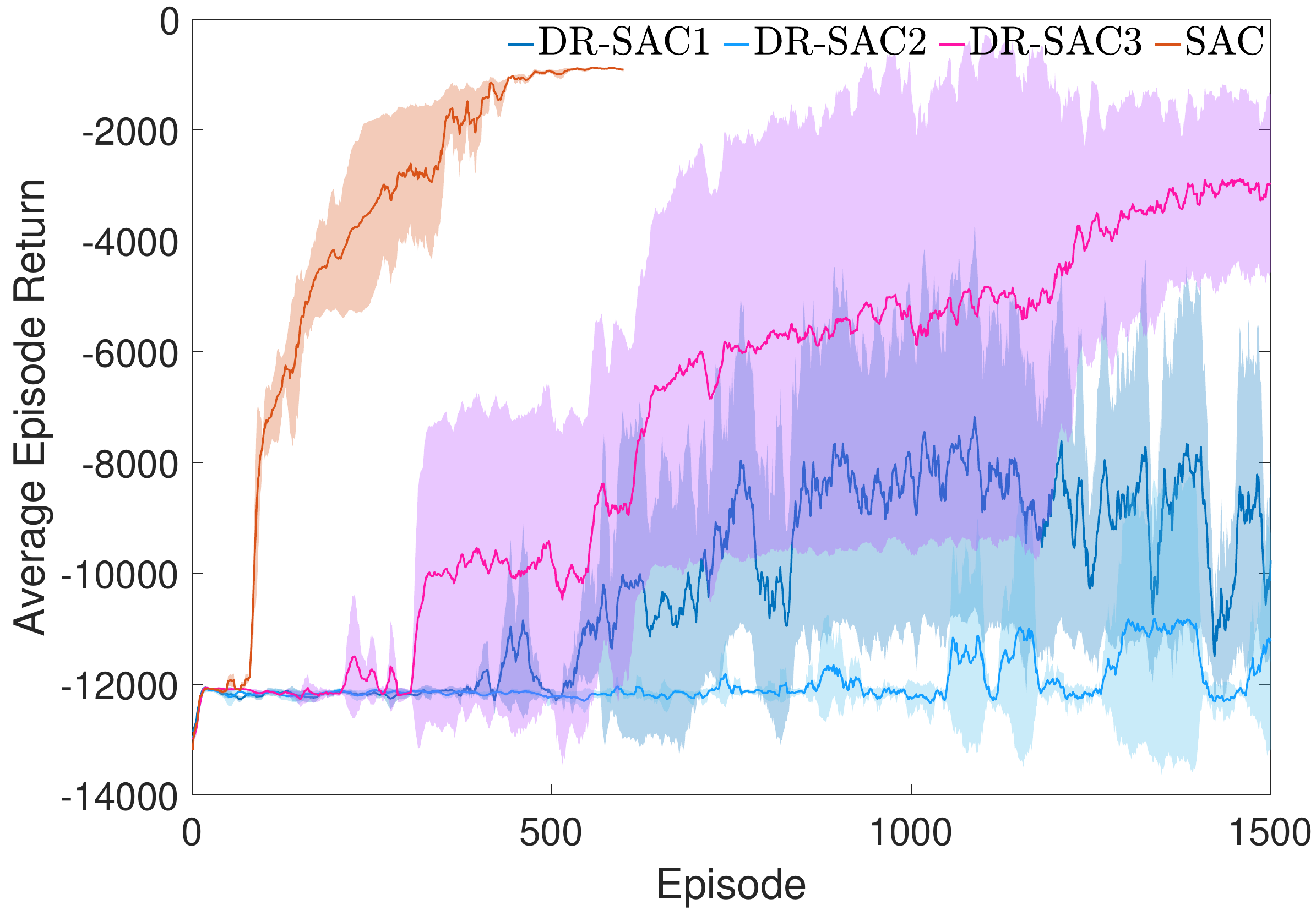}
    \caption{\black{Training curves for Pendubot. Shaded areas denote the variance over five trials.}}
    \label{fig:training-curves}
    \vspace{-2mm}
\end{figure}

\begin{figure*}
\centering
\includegraphics[width=0.8\linewidth]{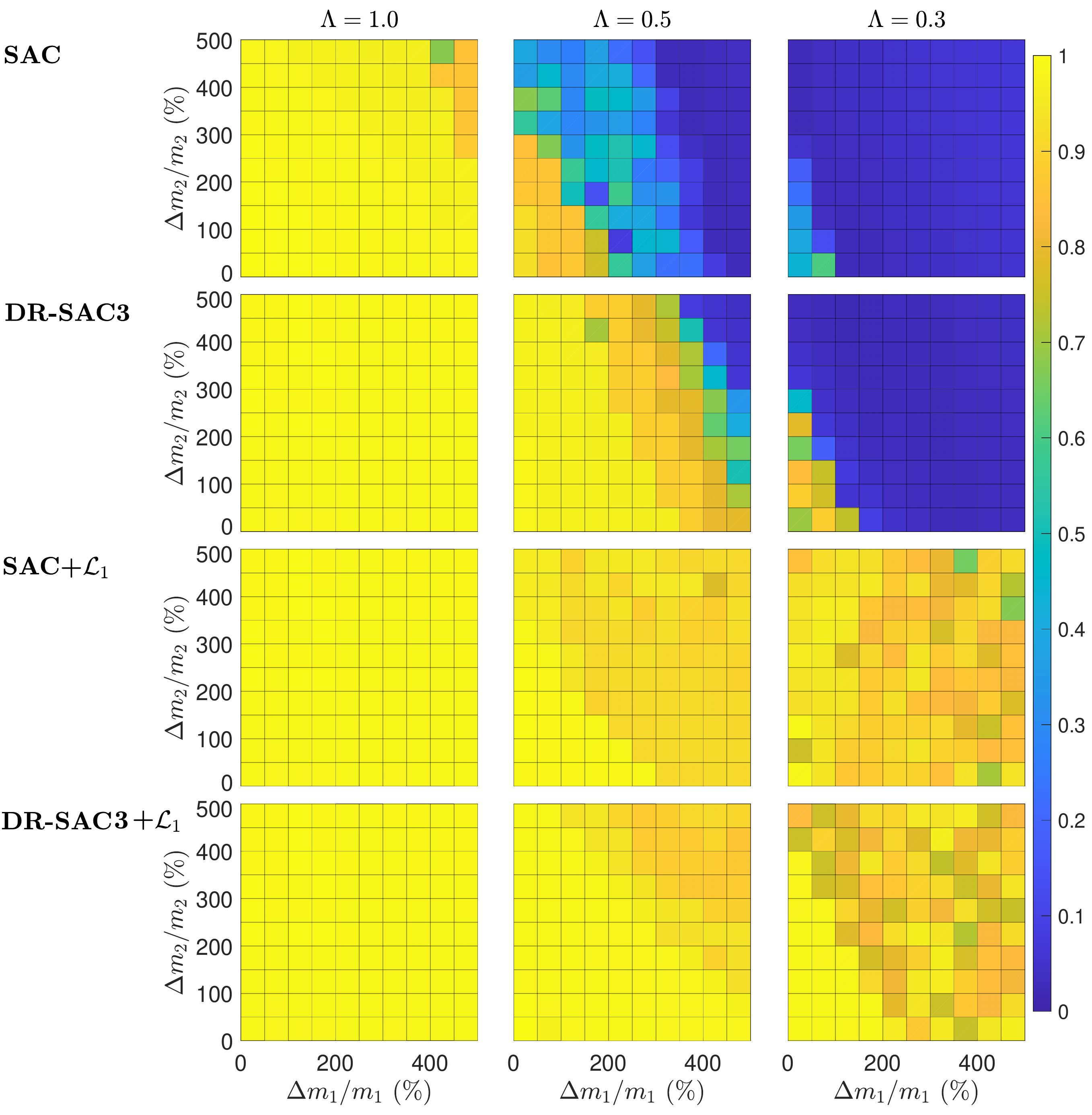}
  \vspace{-2mm}
  \caption{\black{Performance of SAC, DR-SAC3, SAC+\lonew and DR-SAC3+\lonew for Pendubot under perturbations in $m_1$, $m_2$ and $\Lambda$. Percentage change with respect to the nominal value is used to measure the perturbations in $m_1$ and $m_2$.}}
   \label{fig:pend-rewards}
    \vspace{-2mm}
\end{figure*}

The nominal RL policies were trained in simulation using soft actor-critic (SAC) \cite{haarnoja2018soft-sac} implemented in the MATLAB Reinforcement Learning Toolbox.  \black{For comparison, we also trained a few robust policies (termed as DR-SAC) with SAC and domain randomization \cite{peng2018sim2real,loquercio2019deep-drone-RL}, in which three parameters, namely, the input gain ($\Lambda$), the mass of Link 1 ($m_1$), and the mass of Link 2 ($m_2$), were randomly sampled in a variety of 
ranges. Additionally, we tried imposing different control limits (through squashing). When training the SAC and DR-SAC polices, each agent includes an actor and two critics, all three of which share the same neural network structure that has two hidden fully-connected layers with 300 and 400 neurons, respectively. The same hyper-parameters were used for training all the DR-SAC and SAC policies. We did five trials for each setting.
\cref{table:training_policy} lists three of many settings that we tested for training the DR-SAC policies and the setting for training the vanilla SAC policy.
Figure~\ref{fig:training-curves} shows the average episode return (computed using a window of 10 episodes) during training.
The solid curves correspond to the mean and the shaded region to the minimum and maximum average return over the five trials.
As seen in Fig.~\ref{fig:training-curves}, it was much easier and 
took much less episodes to find a good SAC policy, compared to training DR-SAC policies. We were able to find a good DR-SAC policy (i.e., DR-SAC3) under Setting IV, while further increasing the range of parameter perturbations associated with Setting IV led to degraded performance of the resulting DR-SAC policies even with a larger control limit, as illustrated by the training curves for DR-SAC1 and DR-SAC2. For subsequent tests, we chose the best DR-SAC3 from all five trials and compared it with other  control policies. 
}
\begin{figure*}[ht]
   \centering
   \subfloat[]{
   \label{fig:engine_failure}
   \begin{minipage}[c]{0.3\textwidth}
   \centering
   \centerline{\includegraphics[width=0.95\linewidth]{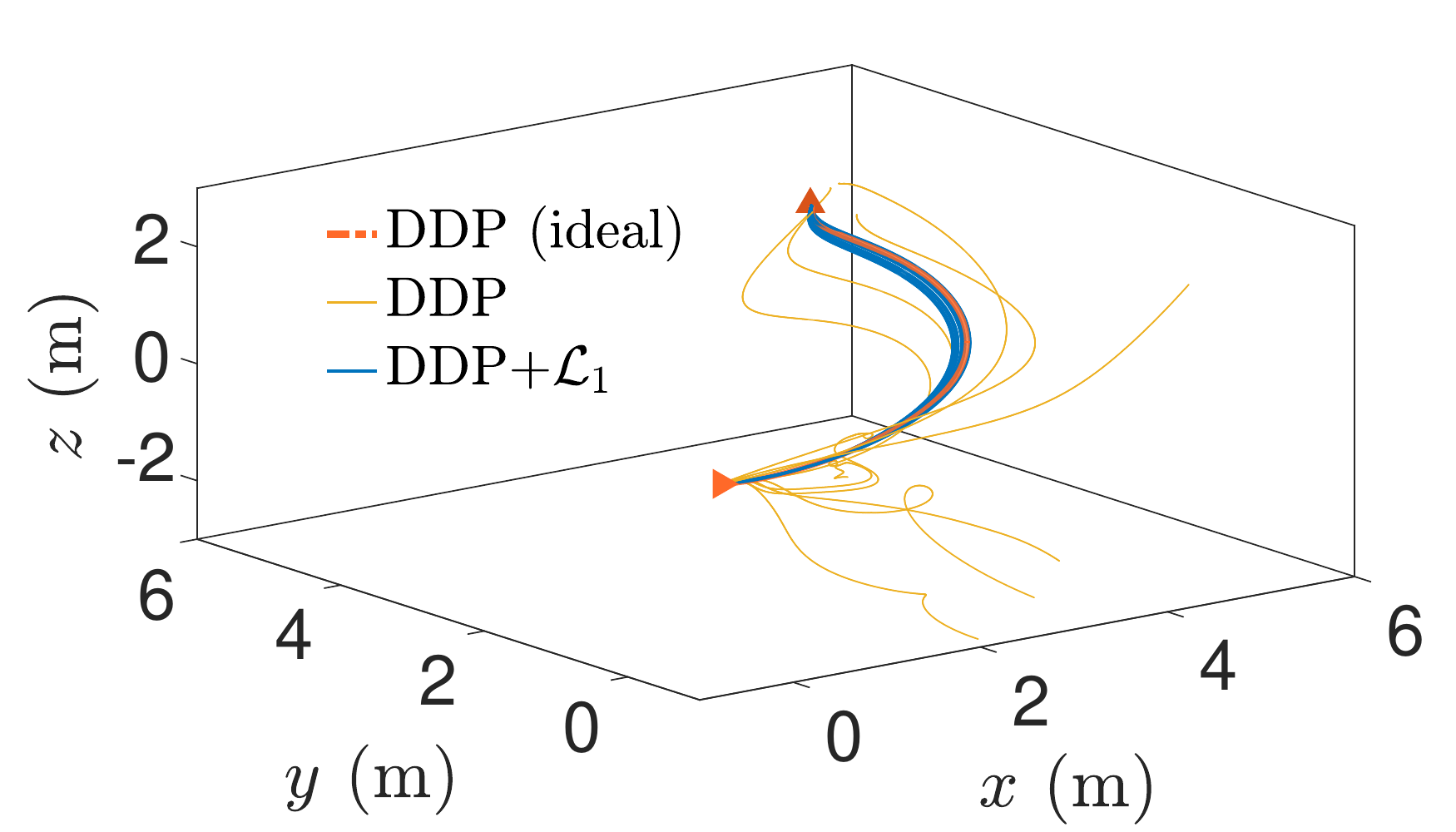}}
   \end{minipage}
   }
   \centering
   \subfloat[]{
   \label{fig:mass_inertia_change}
   \begin{minipage}[c]{0.3\textwidth}
   \centering
   \centerline{\includegraphics[width=0.95\linewidth]{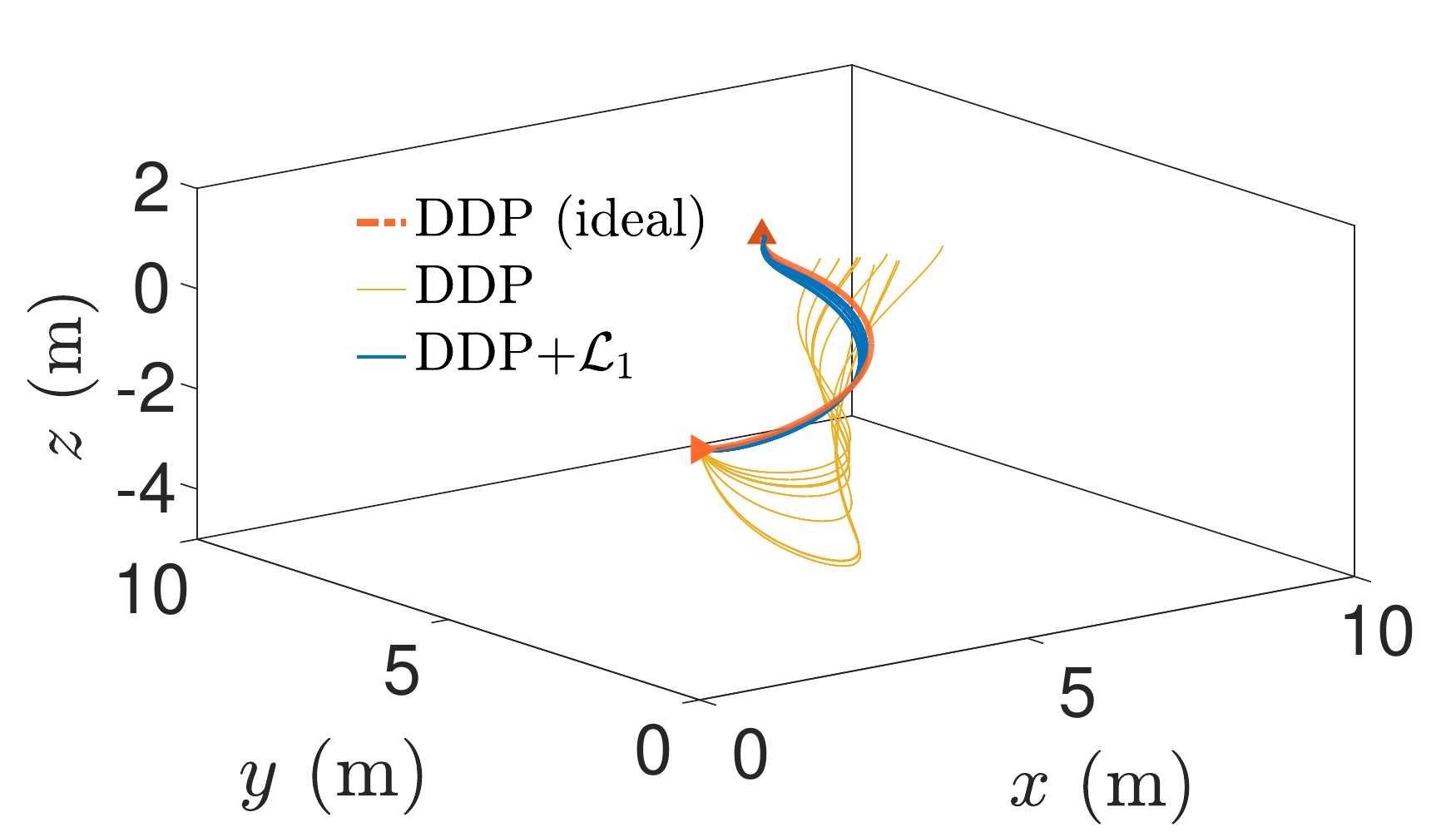}}
   \end{minipage}
   }
   \centering
   \subfloat[]{
   \label{fig:wind_disturbance}
   \begin{minipage}[c]{0.3\textwidth}
   \centering
   \centerline{\includegraphics[width=0.95\linewidth]{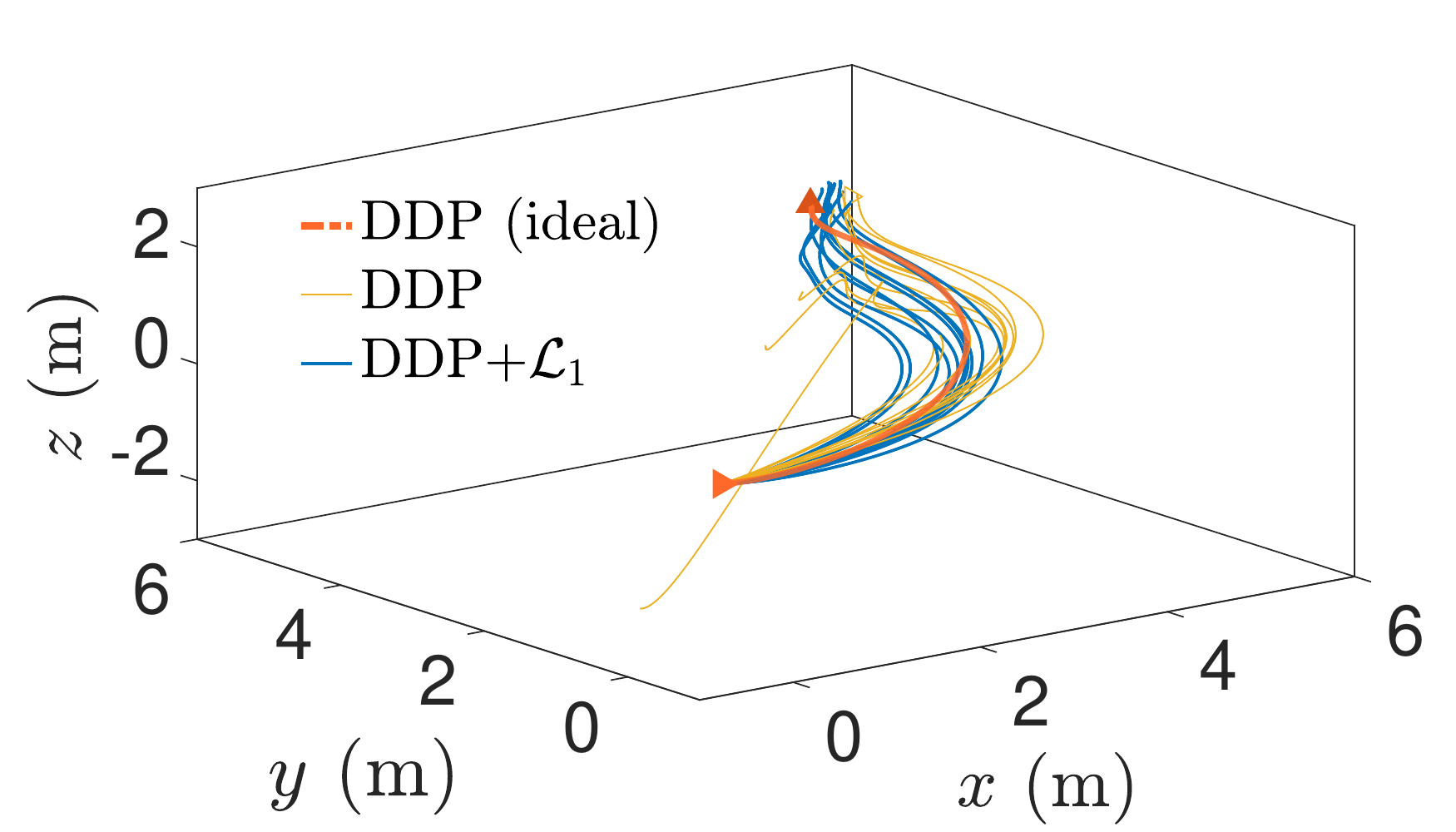}}
   \end{minipage}
   }
   \caption{
  \black{Results under loss of propeller efficiency ((a)), perturbations in quadrotor mass and inertia ((b)), and wind disturbances ((c)). DDP (ideal) denotes the trajectory obtained by applying the policy to the nominal dynamics.}  }
   \label{RT}
   \vspace{-2mm}
\end{figure*}

 \black{We tested the performance of vanilla SAC, DR-SAC, SAC with \lonew augmentation (SAC+$\lone$) and DR-SAC with \lonew augmentation (DR-SAC+$\lone$) under a wide range of perturbations in $m_1$, $m_2$, and under three input gain settings: $\Lambda$ = 1.0, 0.5 and 0.3, while the latter two indicate a loss of control effectiveness by 50\% and 70\%, respectively. For \lonew augmentation design, the parameters in \cref{eq:state-predictor,eq:adaptive_law,eq:adaptive-control-law} were chosen to be $a=10$, $T=0.005$ second and $K=200$, and {\it fixed} across all the tests. The results in terms of the normalized accumulative reward under each test scenario are shown in \cref{fig:pend-rewards}. Note that perturbation in $m_2$ induces unmatched uncertainties that cannot be compensated by the \lonew control law.  As one can see, the performance of vanilla SAC drops dramatically when the perturbations in $m_1$, $m_2$ and $\Lambda$ increase. DR-SAC3 achieved acceptable performance under $\Lambda=0.5$ in general, except when the perturbations in $m_1$ and $m_2$ are near the maximum, which are beyond the perturbations encountered during training of DR-SAC3. However, when the control effectiveness further decreases to 30\% of its nominal value, DR-SAC3's performance degrades significantly, while only slight performance degradation is observed under SAC+$\mathcal{L}_{1}$ and DR-SAC3+$\mathcal{L}_{1}$ 
 when the perturbations increase to the maximum.  It is worth noting that SAC+\lonew and DR-SAC3+\lonew show comparable performance under the tested scenarios. We conjecture that  in the case of larger unmatched uncertainties, DR-SAC3+\lonew will outperform SAC+$\lone$}.

\subsection{\black{3-D quadrotor navigation in simulations}}\label{sec:sub-exp-quadrotor}
The states include quadrotor position ($x,y,z$) and linear velocities ($\dot{x},\dot{y},\dot{z}$) in an inertia frame and the roll, pitch, and yaw angles ($\phi$,$\theta$,$\psi$)
of the quadrotor body frame {with respect to} the inertial
frame, as well as their derivatives. Motor mixing is also included in the dynamics. The inputs are the total thrusts $f_z$ and three moments along three axes ($\tau_{\phi}$,$\tau_{\theta}$,$\tau_{\psi}$) generated by the four propellers. 

The nominal value of the key parameters are set to be $[I_x, I_y, I_z]=[0.082, 0.0845,0.1377]~\mathrm{{\mathrm{kg}m}^2}$ (moment of inertia), $m=4.34\!~\mathrm{kg}$  (quadrotor mass), and $c_{pi}=1$ ($i=1,2,3,4$) (propeller control coefficients). The mission is to control the quadrotor to fly from the origin to the target point $(4,4,2)$. 
To obtain a policy for achieving the mission, we chose to use {trajectory optimization, which, together with model learning, is commonly used for model-based RL} \cite{levine2014learning-trajOpt, pan2014pddp}. We further chose to use differential dynamic programming (DDP) \cite{tassa2012synthesis}, a specific trajectory optimization method. Since our focus is not on the training but on robustifying a pre-trained policy, we use the physics-based dynamic model with the nominal parameter values as the model ``learned'' in the nominal environment. This model is used for computing the DDP policy, and for designing the adaptive augmentation. \black{For computing the DDP policy, we  discretized the nominal dynamics and applied the method in \cite{tassa2012synthesis} with the cost function
 $J = \tilde{x}_{N}^{\top} P_N \tilde{x}_{N}+\sum_{i=0}^{N-1}\left(\tilde{x}_{i}^{\top}P  \tilde{x}_{i}+u_{i}^{\top} Q u_{i}\right),$
where $\tilde{x}_i = x_i - x_{target}$ for $i = 1,...,N$, $N$ is the control horizon, and $P = \textup{diag}(2,2,2,0.1,0.1,0.3,0.1,0.1,0.1,0.1,0.1,0.1)$, $P_N = \textup{diag}(10,10,10,5,5,5,5,5,5,5,5,5)$ and $Q=\textup{diag}(20,4,4,4)$.} For \lonew augmentation design, the parameters in \cref{eq:state-predictor,eq:adaptive_law,eq:adaptive-control-law} were chosen to be $a=10$, $T=0.001$ second and $K=200$, and {\it fixed} across all the tests.

 We tested the performance of the DDP policy with and without \lonew augmentation under three types of dynamic perturbations. The first one is {\bf loss of propeller efficiency}, which mimics the effect of propeller failures, and is simulated by adjusting the control coefficients $c_{pi}$ ($i=1,2,3,4)$. Figure~\ref{fig:engine_failure} shows the  resulting trajectories under ten scenarios, in each of which the control coefficients of two propellers were randomly selected to be in $[0.5,1]$. One can see that \lonew augmentation significantly improved the robustness of the DDP policy, leading to consistent trajectories that are close to the ideal trajectory obtained by applying the policy to the nominal dynamics. 
The second type of dynamic perturbations are the {\bf mass and inertia change}, e.g., to mimic the effect of carrying different packages for a delivery drone. Fig.~\ref{fig:mass_inertia_change} shows the results under ten scenarios with randomly increased mass and inertia through a scale of $[2,5]$. Once again, \lonew augmentation significantly improved the policy robustness, leading to close-to-ideal trajectories. 
The third type of dynamic variations is related to {\bf wind disturbances} in the horizontal plane, which causes disturbance forces in the $x$ and $y$ directions. In each of the ten scenarios, the forces were simulated by stochastic variables with the mean values randomly sampled from $[10,25]$. The results are depicted in Fig.~\ref{fig:wind_disturbance}. \lonew augmentation improved the robustness, but was not able to yield close-to-ideal performance. This is mainly because the wind disturbances induce unmatched disturbances ($\hsigma_{um}(t)$ in \eqref{eq:state-predictor} and \eqref{eq:adaptive_law}), which are not compensated for in the control law \eqref{eq:adaptive-control-law}. \black{Finally, Fig.~\ref{fig:joint-perturb} illustrates the simulation results under {\bf joint perturbations in quadrotor mass, inertia and propeller efficiency and wind disturbances}.}
\begin{figure}[t]
\centering
  \includegraphics[width=0.6\linewidth]{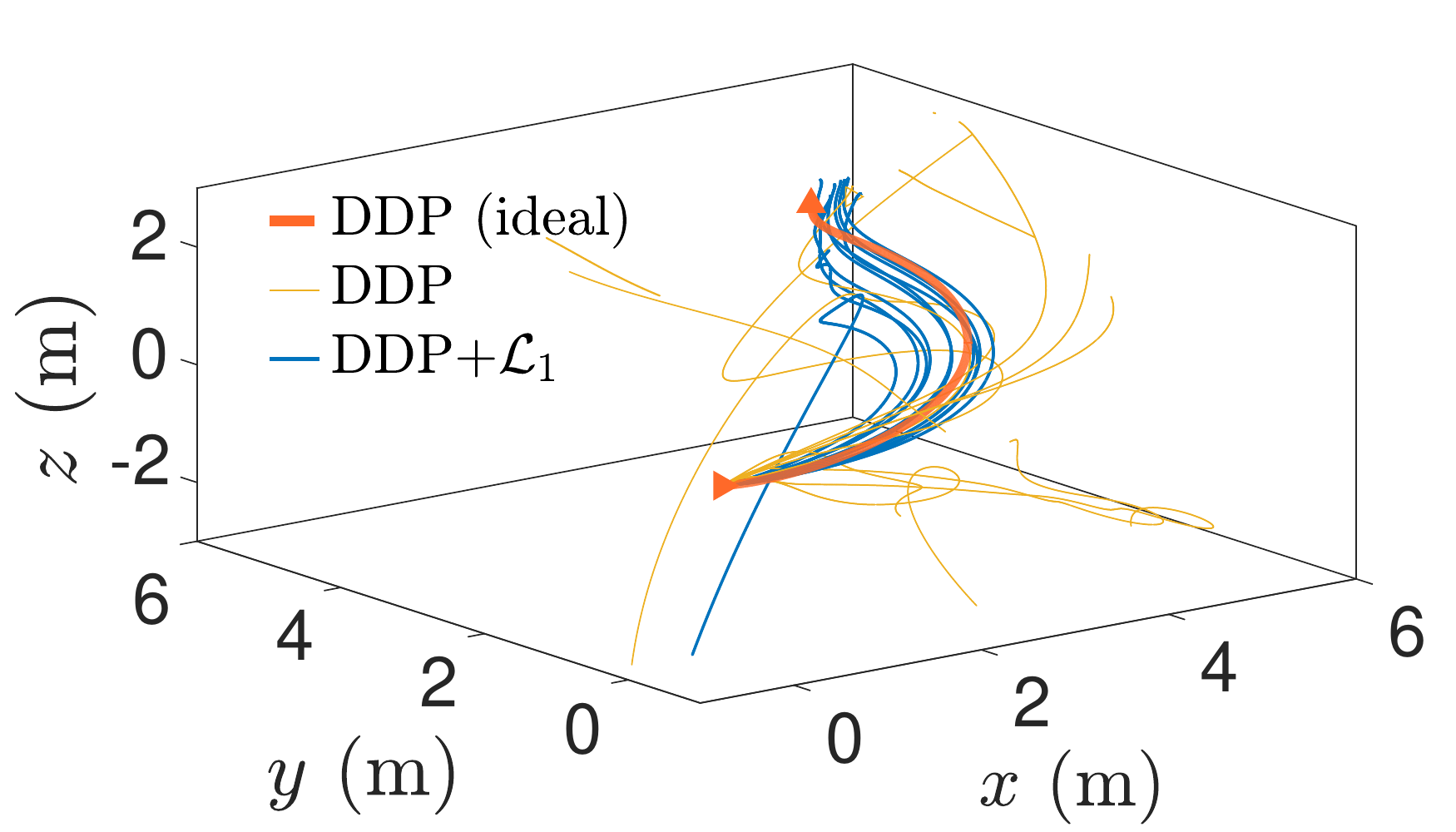}
   \vspace{-2mm}
  \caption{\black{Results under joint perturbations in quadrotor mass, inertia and propeller efficiencies, and wind disturbances. In each of the ten scenarios, each type of perturbation was generated in the same way as was for the results in \cref{fig:engine_failure,fig:mass_inertia_change,fig:wind_disturbance}}.}
  \label{fig:joint-perturb}
  \vspace{-2mm}
\end{figure}
\begin{table*}[htb]
\centering
\caption{Test results under different scenarios}\label{table:test-results-hardware}
\begin{tabular}{|c|c|c|c|c|c|c|}
\hline
\diagbox{Scenario}{Policy}&PILCO&PILCO+$\mathcal{L}_{1}$&SAC&SAC+$\mathcal{L}_{1}$&DR-SAC&DR-SAC+$\mathcal{L}_{1}$ \\
\hline
I: Nominal&{\color{green}\CheckmarkBold}&{\color{green}\CheckmarkBold}&{\color{green}\CheckmarkBold}&{\color{green}\CheckmarkBold}&{\color{green}\CheckmarkBold}&{\color{green}\CheckmarkBold}\\
\hline
II: $\Lambda=0.5$&{\color{red}\XSolidBrush}&{\color{green}\CheckmarkBold}&{\color{red}\XSolidBrush}&{\color{green}\CheckmarkBold}&{\color{green}\CheckmarkBold}&{\color{green}\CheckmarkBold}\\
\hline
III: Added Masses of $270g$ ($100\%$ of $m_2$) &{\color{green}\CheckmarkBold}&{\color{green}\CheckmarkBold}&{\color{green}\CheckmarkBold}&{\color{green}\CheckmarkBold}&{\color{green}\CheckmarkBold}&{\color{green}\CheckmarkBold}\\
\hline
IV: $\Lambda=0.6$ plus Added Mass of $90g$ ($33\%$ of $m_2$) &{\color{red}\XSolidBrush}&{\color{green}\CheckmarkBold}&{\color{red}\XSolidBrush}&{\color{green}\CheckmarkBold}&{\color{green}\CheckmarkBold}&{\color{green}\CheckmarkBold}\\
\hline
V:  $\Lambda=0.5$ plus Added Masses of $450g$ ($167\%$ of $m_2$)&{\color{red}\XSolidBrush}&{\color{green}\CheckmarkBold}&{\color{red}\XSolidBrush}&{\color{red}\XSolidBrush}&{\color{red}\XSolidBrush}&{\color{green}\CheckmarkBold}\\
\hline
VI: Added disturbances with a rubber band&{\color{red}\XSolidBrush}&{\color{green}\CheckmarkBold}&{\color{red}\XSolidBrush}&{\color{green}\CheckmarkBold}&{\color{green}\CheckmarkBold}&{\color{green}\CheckmarkBold}\\
\hline
\end{tabular}
\label{tab:experiments_scenario}
\end{table*}
\subsection{Pendubot swing-up and balance on  real hardware} \label{sec:sub-exp-pendubot-realworld}
We further tested the performance of those policies used in \cref{sec:sub-exp-pendubot-sim} on the hardware setup depicted in \cref{fig:pend_demo}.
\begin{figure}[h] 
\vspace{-1mm}
\centering \includegraphics[height=0.3\linewidth]{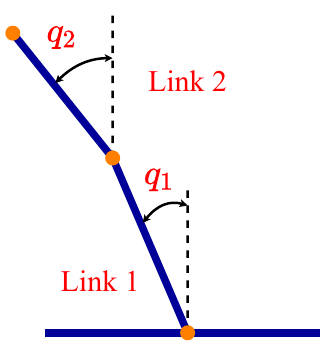}\includegraphics[height=0.51\linewidth]{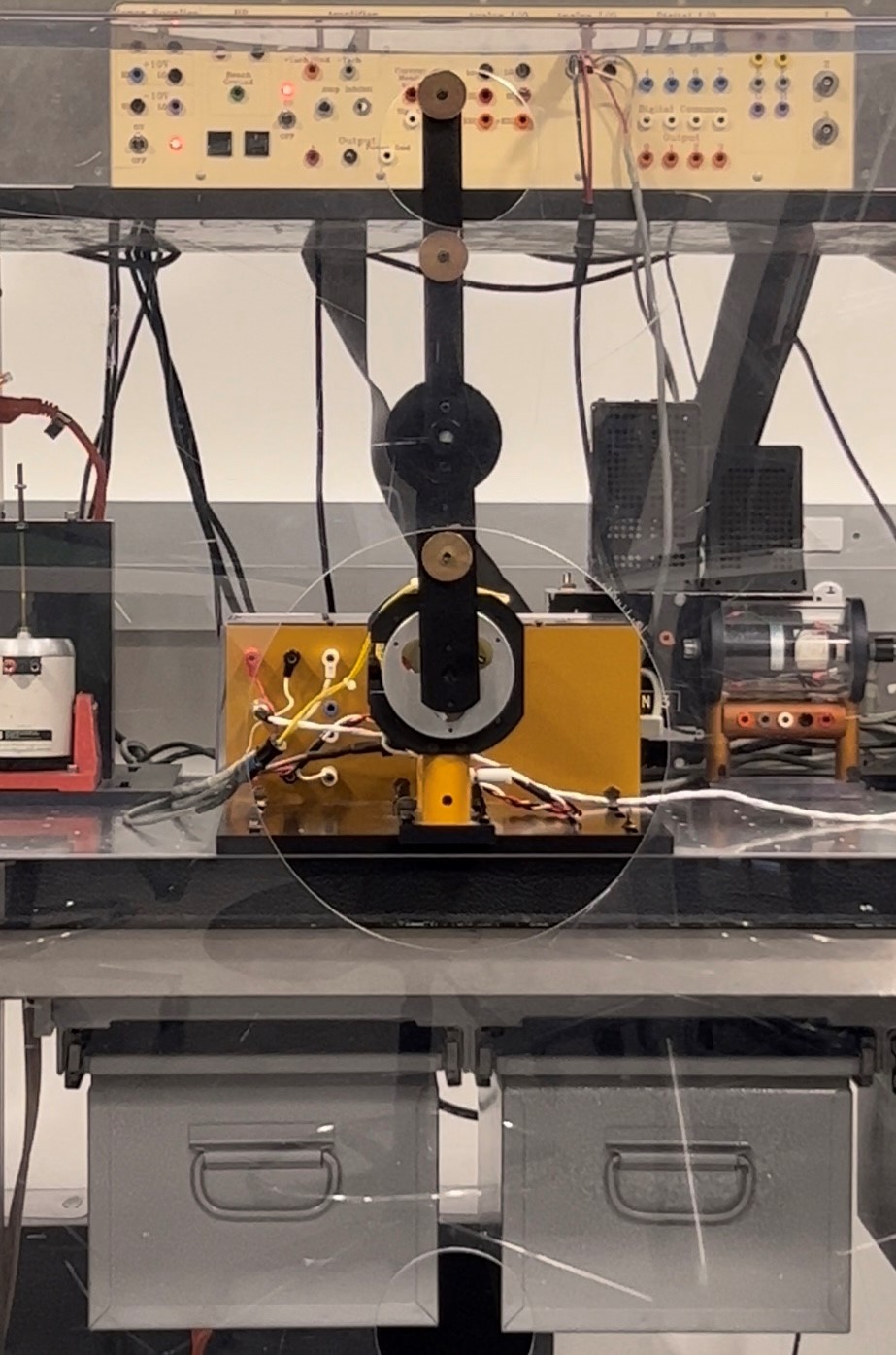} \includegraphics[height=0.51\linewidth]{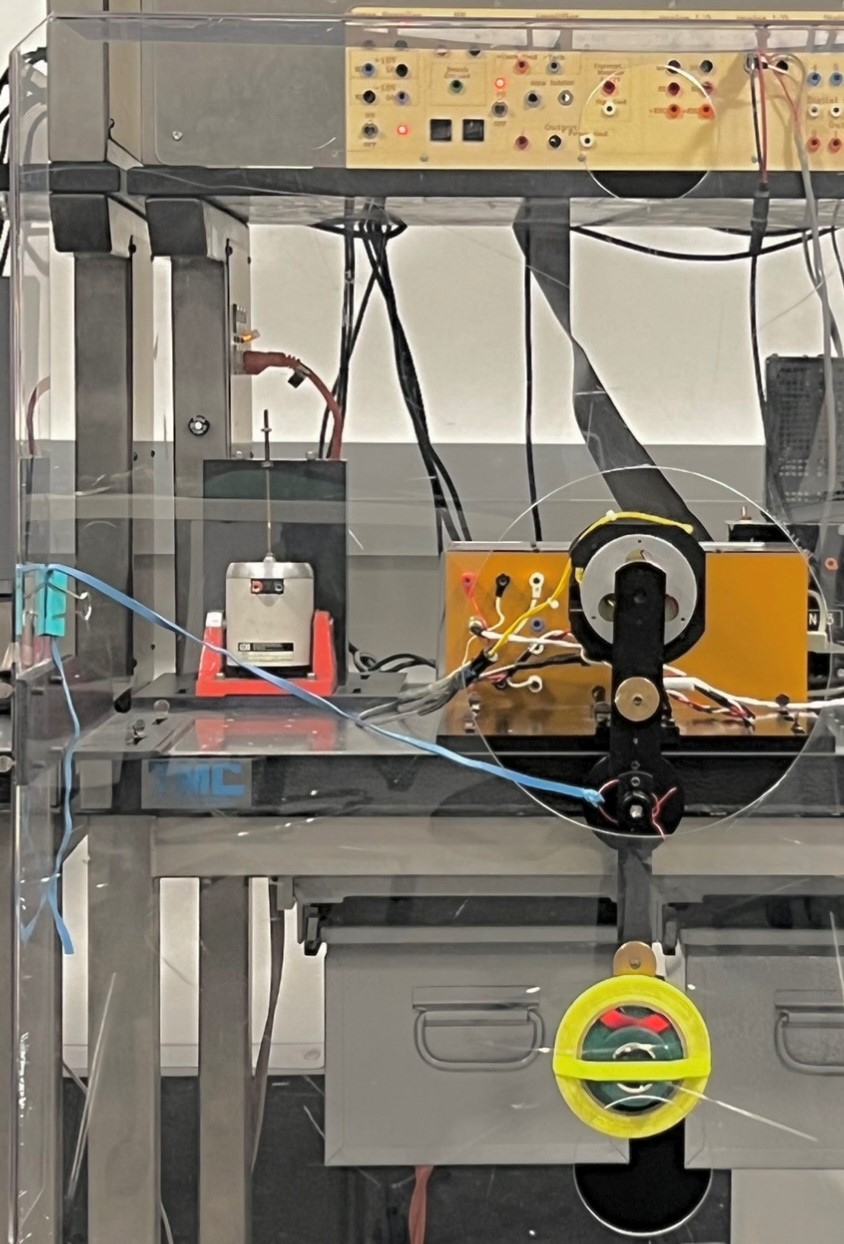}
  \caption{\black{Left: a Pendubot configuration. Middle: stabilization at the upright position. Right: added masses and a rubber band used to induce dynamic variations}}  \label{fig:pend_demo}
  \vspace{-5mm}
\end{figure}
In addition to SAC and DR-SAC, we trained another policy using PILCO \black{with the same reward function defined by \cref{pend_reward_function}}. The ways to introduce dynamic variations include changing the input gain $\Lambda$, adding masses to Link~2, adding disturbance forces using a rubber band and different combinations of these three ways. {For \lonew augmentation design, the parameters in \cref{eq:state-predictor,eq:adaptive_law,eq:adaptive-control-law} were chosen to be $a=150$, $T=0.005$ s and $K=150$ for most of the policies in most of the test scenarios. For DR-SAC in Test I, $K=100$ (corresponding to a lower bandwidth for the low-pass filter) was used to avoid large vibrations at the upright position, due to the fact that DR-SAC has a relatively high gain to attenuate the effect of dynamic variations. }
The test scenarios and results are summarized in \cref{table:test-results-hardware}, where 
{\color{green}\CheckmarkBold} ({\color{red}\XSolidBrush}) indicates a success (failure) in achieving the mission. \black{A video of the experiments is available at \url{https://youtu.be/xZBcsNMYK3Y}}.


As one can see, in the nominal case (i.e., without intentionally introduced dynamic variations), all the policies with and without \lonew augmentation succeeded in achieving the mission. This, to a certain extent, indicates that the \lonew augmentation does not adversely affect the performance of RL policies in the presence of no or minimal dynamic variations. Additionally, \lonew augmentation significantly improves the robustness of PILCO and vanilla SAC, enabling them to succeed under all the tested scenarios except Scenario~V for SAC, due to the extreme dynamic variations induced by the the largest perturbations in input gain and added masses.   
DR-SAC displayed much more robustness compared to vanilla SAC as expected, and only failed under Scenario~V. It's worth noting that \lonew augmentation also further enhanced the robustness to DR-SAC and made it succeed under Scenario~V. In Scenario~VI, a rubber band was attached to the joint connecting the two links to exert a disturbance force. The disturbance force applied by the rubber band changed quite rapidly and peaked when Link 1 reached the upright position. This caused great challenges for the RL policies, as evidenced by the struggling of  PILCO and SAC in the video,  since, by training, these policies are not expected to produce large control inputs near the upright position. Nevertheless, with the help of \lonew compensation,  PILCO and SAC were able to deal with this challenging scenario. 

\section{Conclusion}
This paper presents an add-on scheme to improve the robustness of a reinforcement learning (RL) policy for controlling systems with continuous state and action spaces, by augmenting it with an \lonew adaptive controller (\loneAC) that can quickly estimate and actively compensate for potential dynamic variations during execution of this policy. Our approach is easy to implement and allows for the policy to be trained or computed using almost any RL method (model-free or model-based), either in a simulator or in the real world, as long as a control-affine model to describe the dynamics of the nominal environment is available for the \loneAC~design. Experiments on different systems in both simulations and on real hardware demonstrate  the general applicability of the proposed approach and its capability in improving the robustness of RL policies including those trained robustly, e.g., using domain/dynamics randomization (DR). Future work includes incorporating mechanisms, e.g., based on robust control \cite{zhao2021RALPV-tac,zhao2022tube-rccm-ral}, to mitigate the effect of unmatched disturbance, and model learning to safely and robustly learn the unknown dynamics. 

The proposed approach and existing robust RL methods e.g., based on DR, do not necessarily replace each other. Instead, they can complement each other, as demonstrated by the experimental results in Section~\ref{sec:sub-exp-pendubot-realworld}. As mentioned before, existing robust RL methods \black{aim to find a {\it fixed} policy for all possible realizations of uncertainties, which could be infeasible when the range of  uncertainties is large. On the other hand, the proposed adaptive augmentation approach can deal with significant amount of matched uncertainties by using additional control effort to actively compensate for those, but cannot handle unmatched uncertainties in its current form}. For systems subject to both matched and unmatched disturbances, a compelling solution will be to combine the strength of both by (1) (partially) ignoring matched disturbances in training a policy using existing robust RL methods to reduce conservativeness, and (2) augmenting the trained policy with the proposed \lonew scheme during execution of this policy to compensate for matched disturbances.


\appendix


\subsection{Proof of Lemmas}
Hereafter, the notations 
 $\mbZ_i$ and $\mbZ_1^n$ denote the integer sets $\{i, i+1, i+2, \cdots\}$ and $\{1, 2,\cdots,n\}$, respectively. 
\subsubsection{Proof of Lemma~\ref{lemma:d-bound}} \label{sec:sub-proof-lemma-d-bound}
Note that $\norm{d(t,x)} =   \norm{d(t,x) - d(t,0) + d(t,0)}   \leq l_d \norm{x} +  \norm{d(t,0)}    \leq  l_d \max_{x\in \mcX}\norm{x} + b_d = \theta$, for any $t$ in $[0,\tau]$, 
where the two inequalities hold due to \cref{eq:d-lipschitz-cond} and \cref{eq:d-x0-bound} in Assumption~\ref{assump:lipschitz-bound-fg}. 
Additionally, 
 $    \norm{g(x)(\Lambda-I) u} \leq  \norm{g(x)}\norm{(\Lambda-I)}\norm{u} \leq  \rho$  for any $t$ in $[0,\tau]$.  Therefore, $\norm{\sigma(t,x,u)}\leq  \norm{g(x)(\Lambda-I) u} + \norm{d(t,x)}\leq \rho+\theta$  for any $t$ in $[0,\tau]$. Furthermore,  the dynamics  \eqref{eq:dynamics-perturb-wrt-nom} implies that 
$
   \norm{\dotx(t)}\leq \norm{f(x)+g(x)u} + \norm{g(x)(\Lambda-I) u} +  \norm{d(t,x)} \leq \max_{x\in \mcX}\norm{f(x)}+\max_{x\in \mcX}\norm{g(x)}\max_{u\in \mcU}\norm{u}  +\rho+\theta =  \phi,
$   for any $t$ in $[0,\tau]$. The proof is complete. \qed 

\subsubsection{Proof of Lemma~\ref{lemma:estiamte-error-bound}} \label{sec:sub-proof-lemma-estimate-err-bnd}
From \eqref{eq:dynamics-perturb-wrt-nom} and \eqref{eq:state-predictor}, the prediction error dynamics are obtained as
\begin{equation}
    \dot{\tilx}(t)=\hat{\sigma}(t) - \sigma(t,x,u) -a\tilx(t), \quad \tilx(0) = 0. 
\end{equation}
Therefore, $\hat{\sigma}(t) = 0 $ for any $t\in[0,T)$ according to \eqref{eq:adaptive_law}. Further considering the bounds on $d(t,x)$ and $g(x)(\Lambda-I) u$ in \eqref{eq:d-xdot-bound}, we have
\begin{equation}
    \begin{split}\label{eq:estimation-error-bound-0-T}
      \norm{\hat{\sigma}(t)-\sigma(t,x,u)}   &= \norm{\sigma(t,x,u)}\\
        &\leq \norm{d(t,x)}+\norm{g(x)(\Lambda-I)u}\\
        & \leq \theta + \rho, \quad \forall t\in[0,T).
    \end{split}
\end{equation}
We next derive the bound on $\norm{\hat{\sigma}(t)-\sigma(t,x,u)}$ for $t\in [T, \tau)$. For notation brevity, hereafter, we often write $\sigma(t,x,u)$ as $\sigma(t)$, $d(t,x)$ as $d(t)$, and $g(x)(\Lambda-I)u(t)$ as $h(t)$, i.e., $h(t)\trieq g(x(t))(\Lambda-I)u(t)$. For any $t\in [iT, (i+1)T)\cap[0,\tau]$ ($i\in \mbZ_0$), we have
{
$$
    \tilx(t) = e^{-a(t-iT)}\tilx(iT) + \int_{iT}^t e^{-a(t-\xi)}(\hat{\sigma}(\xi)-\sigma(\xi)))d\xi.
$$}
Since $\tilx(t)$ is continuous, the preceding equation implies 
{\begin{align}
  \tilx((i+1)T) \nonumber =  & ~  e^{-aT}\tilx(iT) \! +\!\int_{iT}^{(i+1)T}\!\! e^{-a((i+1)T-\xi)}d\xi \hd(iT) \nonumber \\
    &  - \int_{iT}^{(i+1)T} e^{-a((i+1)T-\xi)}\sigma(\xi)d\xi, \nonumber \\
     = &~e^{-aT}\tilx(iT) + \frac{1-e^{-aT}}{a}\hat{\sigma}(iT)   \nonumber\\
     &~- \int_{iT}^{(i+1)T} e^{-a((i+1)T-\xi)}\sigma(\xi)d\xi,  \nonumber\\
    =& ~ - \int_{iT}^{(i+1)T} e^{-a((i+1)T-\xi)}\sigma(\xi)d\xi,  \label{eq:tilx-iplus1-Ts}
\end{align}}where the first and last equalities are due to the estimation law \eqref{eq:adaptive_law}. 
%
Since $x(t)$ is continuous,
$d(t,x)$ and $h(t)$ are also continuous, given \cref{assump:nominal-policy,assump:lipschitz-bound-fg}, and the control law \cref{eq:adaptive-control-law}. Therefore, $\sigma(t)$ is continuous.  Furthermore, considering that  $e^{-a((i+1)T-\xi)}$ is always positive, we can apply the first mean value theorem in an element-wise manner\footnote{{Note that the mean value theorem for definite integrals 
 only holds for scalar valued functions.}} to \eqref{eq:tilx-iplus1-Ts}, which leads to 
{\begin{align}
    \tilx((i+1)T) = & -\int_{iT}^{(i+1)T} e^{-a((i+1)T-\xi)}d\xi \bbracket{\sigma_j(\xi^*_{j})}, \nonumber\\
    =& -\frac{1}{a}(1-e^{-aT})\bbracket{\sigma_j(\xi^*_{j})}, \label{eq:xtilde-iplus1-Ts-d-tau}
\end{align}}for some $\xi^*_{j}\in (iT, (i+1)T)$ with $j\in\mbZ_1^n$ and $i\in \mbZ_0$, where $\sigma_{j}(t)$ is the $j$-th element of $\sigma(t)$, and
{$$\bbracket{\sigma_j(\xi_j^*)}\triangleq 
[\sigma_1(\xi_1^*),\cdots, \sigma_n(\xi_n^*)
]^\top.$$}
The adaptive law \eqref{eq:adaptive_law} indicates that for any $t\in[(i+1)T, (i+2)T))\cap[0,\tau]$, we have
$\hat{\sigma}(t) = -\frac{a}{e^{aT}-1}\tilx((i+1)T))$. 
The preceding equality and \eqref{eq:xtilde-iplus1-Ts-d-tau} imply that for any $t\in[(i+1)T, (i+2)T)$ with $i\in \mbZ_0$, there exist $\xi_j^*\in(iT,(i+1)T))\cap[0,\tau]$ ($j\in\mbZ_1^n$) such that 
\begin{equation}\label{eq:hatd-d-tau-star-relation}
    \hat{\sigma}(t) = e^{-aT}\bbracket{\sigma_j(\xi_j^*)}.
\end{equation}
Note that 
{ 
\begin{align}
 &    \norm{\sigma(t)-\bbracket{\sigma_j(\xi_j^*)}} \leq \sqrt{n}\infnorm{\sigma(t)-\bbracket{\sigma_j(\xi_j^*)}} \nonumber\\
    = & \sqrt{n}\abs{\sigma_{\bar j_t}(t)-{\sigma_{\bar j_t}(\xi_{\bar j_t}^*)}} 
    \leq \sqrt{n}\norm{\sigma(t)-{\sigma(\xi_{\bar j_t}^*)}},\label{eq:d-dStar-index-conversion}
\end{align}
}where $\bar j_t=\arg\max_{j\in\mbZ_1^n} \abs{\sigma_j(t)-{\sigma_j(\xi_j^*)}}$.  Similarly,
{\begin{align}
 \norm{\bbracket{\sigma_j(\xi_j^*)}} &\leq \sqrt{n}\infnorm{\bbracket{\sigma_j(\xi_j^*)}} 
      =   \sqrt{n}\abs{{\sigma_{{\bar j}}(\xi_{{\bar j}}^*)}}\nonumber  \\
    & \leq \sqrt{n}\norm{{\sigma(\xi_{{\bar j}}^*)}} \leq \sqrt{n}( \theta+\rho) ,  \label{eq:dStar-index-conversion}
\end{align}}where ${\bar j}=\arg\max_{j\in\mbZ_1^n} \abs{{\sigma_j(\xi_j^*)}}$, and the last inequality is due to \cref{eq:d-xdot-bound}.  
Therefore, for any $t\in[(i+1)T, (i+2)T)\cap [0,\tau]$ ($i\in \mbZ_0$), we have
{\begin{align}
  &  \norm{\sigma(t) -\hat{\sigma}(t)}  = \norm{\sigma(t) - e^{-aT}[\sigma_j(\xi_j^*)]} \nonumber\\ 
\leq & \norm{\sigma(t) -[\sigma_j(\xi_j^*)]}  + (1-e^{-aT})\norm{[\sigma_j(\xi_j^*)]} \nonumber \\ 
\leq & \sqrt{n}\norm{\sigma(t)-{\sigma(\xi_{\bar j_t}^*)}} + (1-e^{-aT})\sqrt{n}(\theta+\rho) \nonumber \\
\leq & \sqrt{n}\norm{d(t)-{d(\xi_{\bar j_t}^*)}}  +\sqrt{n}\norm{h(t)-{h(\xi_{\bar j_t}^*)}} \nonumber \\
& +\sqrt{n}(1-e^{-aT})(\theta+\rho),
\label{eq:d-sigmahat-bound}
\end{align}}for some $\xi_{\bar j_t}^*\in (iT,(i+1)T)\cap[0,\tau]$ and  
$\xi_{{\bar j}}^*\in (iT,(i+1)T)\cap[0,\tau]$, where the equality is due to \eqref{eq:hatd-d-tau-star-relation}, and the second inequality is due to \eqref{eq:d-dStar-index-conversion} and \eqref{eq:dStar-index-conversion}.

The inequality \eqref{eq:d-xdot-bound} implies that
{
\begin{align}
    \norm{x(t)-x(\xi_{\bar j_t}^*)} & \leq \int_{\xi_{\bar j_t}^*}^t \norm{\dot{x}(
   \xi)}d\xi \nonumber\\
   &\leq \int_{\xi_{\bar j_t}^*}^t\phi d\xi = \phi(t-\xi_{\bar j_t}^*)\label{eq:x-diff}, 
\end{align}}which, together with \eqref{eq:d-lipschitz-cond}, indicates that 
{\begin{align}
 &   \norm{d(t,x(t))   - d(\xi_{\bar j_t}^*,x(\xi_{\bar j_t}^*))} \nonumber \\
   \leq & l_d^\prime(t-\xi_{\bar j_t}^*) + l_d\norm{x(t)-x(\xi_{\bar j_t}^*)} \nonumber \\
     \leq &  (l_d^\prime+l_d\phi)(t-\xi_{\bar j_t}^*)
   =  \eta_1 (t-\xi_{\bar j_t}^*) \leq 2\eta_{1} T. \label{eq:d-t-taustar-bound}
\end{align}}
To further derive the bound for $\norm{h(t)   - h(\xi_{\bar j_t}^*)}$, we have to prove that the control input $u(t)$ in \eqref{eq:adaptive-control-law} has bounded rate of variation in $[0,\tau]$. Since the nominal policy $\pi_0(x)$ is Lipschitz continuous with a Lipschitz constant $l_\pi$ in $\mcX$ according to \cref{assump:nominal-policy}, $\uRL(t)=\pi_0(x(t))$, and furthermore, $\norm{\dot x(t)}\leq \phi $ in $[0,\tau]$ according to \cref{eq:d-xdot-bound}, we have that $\uRL(t)$ has a bounded rate of variation, $l_{\pi_0}^\prime = l_\pi\phi$, in $[0,\tau]$. 
We next prove $\uLone(t)$ has a bounded rate of variation in $[0,\tau]$. From \cref{eq:adaptive-control-law}, we have $\uLone(s) = -{K}{(sI+K)}^{-1}\hat{\sigma}_{m}(s)$, which indicates $
        \dot u_\lone(t) = -K\uLone(t)-K\hat{\sigma}_{m}(t)$. As a result, 
\begin{equation}
    \norm{\dot u_\lone(t)} \leq \norm{K}\max_{u\in \mcU}\norm{u}+\norm{K}\norm{\hat{\sigma}_{m}(t)}\label{eq:u-bound1}.
\end{equation}Note that $\hat{\sigma}_{m} = g^{+}(x)\hat{\sigma}(t)$ according to \cref{eq:adaptive_law}, where $g^{+}(x)$ is the pseudo inverse of $g(x)$. Then the inequality~\eqref{eq:u-bound1} can be further written as:
\begin{equation}
    \norm{\dot u_\lone(t)} \leq \norm{K}(\max_{u\in \mcU}\norm{u}+\max_{x\in \mcX}\norm{g^{+}(x)}\norm{\hat{\sigma}(t)})\label{eq:u-bound2}.
\end{equation}
The equations in \eqref{eq:hatd-d-tau-star-relation} and \eqref{eq:dStar-index-conversion} indicate that
$ \norm{\hat{\sigma}(t)} \leq  {e^{-aT}}\sqrt{n}(\theta+\rho)
$, which, together with \eqref{eq:u-bound2}, leads to 
\begin{align}
        \norm{\dot u_\lone(t)} & \leq  \norm{K}(\max_{u\in \mcU}\norm{u}+\sqrt{n}e^{-aT}(\theta+\rho) \max_{x\in \mcX}\norm{g^{+}(x)}) \nonumber\\
        & \trieq l_{\lone}^\prime. \label{eq:def_lu}
\end{align} We have proved that $\uLone(t)$ has a bounded rate of variation, $l_{\lone}^\prime$. As a result, $u(t)$ has a bounded rate of variation, $l_{u} ^\prime= l_{\pi_0}^\prime +l_{\lone}^\prime$, as defined in \cref{eq:l_u-defn}, in $[0,\tau]$, i.e., for any $t_1,t_2$ in $[0,\tau]$, we have \begin{equation}\label{eq:u_diff_ineq}
    \norm{u(t_1)-u(t_2)} \leq l_u^\prime \abs{t_1-t_2}. 
    \end{equation}

Now, we have
\begin{align}
    &   \norm{h(t)   - h(\xi_{\bar j_t}^*)} \nonumber\\
     =& \norm{g(x)(\Lambda-I)u(t) - g(x(\xi_{\bar j_t}^*))(\Lambda-I)u(\xi_{\bar j_t}^*)} \nonumber\\
     \leq &\norm{g(x)(\Lambda-I)u(t) - g(x(\xi_{\bar j_t}^*))(\Lambda-I)u(t)} \nonumber\\
     & +\norm{g(x(\xi_{\bar j_t}^*))(\Lambda-I)u(t) - g(x(\xi_{\bar j_t}^*))(\Lambda-I)u(\xi_{\bar j_t}^*)}\nonumber \\
     \leq&\norm{g(x)- g(x(\xi_{\bar j_t}^*))}\norm{\Lambda-I}\norm{u(t)}\nonumber \\
     & + \norm{g(x(\xi_{\bar j_t}^*))}\norm{\Lambda-I}\norm{u(t)-u(\xi_{\bar j_t}^*)}\nonumber\\
     \leq&l_{g}\norm{x-x(\xi_{\bar j_t}^*)}\norm{\Lambda-I}\norm{u(t)} \nonumber\\
     &+ l_{u} ^\prime\abs{t-\xi_{\bar j_t}^{*}}\norm{\Lambda-I}\norm{g(x(\xi_{\bar j_t}^*))}\label{eq:h-bnd-deriv1}\\
     \leq &(l_{g}\phi\norm{u(t)} + l_{u} ^\prime\norm{g(x(\xi_{\bar j_t}^*))})\norm{\Lambda-I}(t-\xi_{\bar j_t}^{*}) \label{eq:h-bnd-deriv2}\\
     \leq& \eta_{2}(t-\xi_{\bar j_t}^{*}) \leq 2\eta_{2}T, \label{eq:h-bnd}
\end{align}
where \cref{eq:h-bnd-deriv1} is due to  \eqref{eq:g-lipschitz-cond} and \eqref{eq:u_diff_ineq}, \cref{eq:h-bnd-deriv2} is due to \cref{eq:x-diff}, 
$\eta_{1}$ and $\eta_{2}$ are defined in \eqref{eq:eta1-defn}, and the last inequality is due to the fact that $t\in[(i+1)T, (i+2)T)$ and  $\xi_{\bar j_t}^*\in (iT, (i+1)T)$. 
Finally, plugging \cref{eq:d-t-taustar-bound,eq:h-bnd} into \eqref{eq:d-sigmahat-bound} leads to 
{\begin{align}
   & \norm{\sigma(t,x,u)-\hat{\sigma}(t)} \nonumber \\
   \leq&  2\sqrt{n}(\eta_{1}+\eta_{2}) T +\sqrt{n} (1- e^{-aT})(\theta+\rho) = \gamma(T), \label{eq:estimation-error-bound-T-inf}
\end{align}}for any $t\in [T, \tau)$, where the second inequality is due to \eqref{eq:d-xdot-bound}. From \eqref{eq:estimation-error-bound-0-T}
and \eqref{eq:estimation-error-bound-T-inf} we arrive at \eqref{eq:estimation_error_bound}.
Considering that $\mcX$ and $\mcU$ are compact, the constants $\theta$ (defined in \eqref{eq:theta-defn}), $\phi$ (defined in \eqref{eq:phi-defn})  and $\eta$ (defined in \eqref{eq:eta1-defn}) are all finite, the definition of $\gamma(T)$ in \eqref{eq:gammaTs-defn} immediately implies    $\lim_{T\rightarrow 0} \gamma(T) = 0$. \qed

\subsection{Dynamic equations used in numerical experiments}\label{sec:sub-dyn-model-systems-experiments}
\subsubsection{Cart-pole} The dynamics of the cart-pole system, taken from \cite{deisenrothPILCOintro} ,is given by
{\setlength{\arraycolsep}{3pt}
\begin{align*}
    \left[\begin{array}{c}
    \dot{x}\\
    \ddot{x}\\
    \ddot{\theta}\\
    \dot{\theta}
    \end{array}\right]\! =\! &
    \left[\begin{array}{c}
    x_{2} \\
    \frac{2 m_{2} l x_{3}^{2} \sin x_{4}+3 m_{2} g \sin x_{4} \cos x_{4}-4 b x_{2}}{4\left(m_{1}+m_{2}\right)-3 m_{2} \cos ^{2} x_{4}} \\
    \frac{-3 m_{2} l x_{3}^{2} \sin x_{4} \cos x_{4}-6\left(m_{1}+m_{2}\right) g \sin x_{4}+6b x_{2} \cos x_{4}}{4 l\left(m_{1}+m_{2}\right)-3 m_{2} l \cos ^{2} x_{4}} \\
    x_{3}
    \end{array}\right] \\ 
    &+
    \left[\begin{array}{c}
    0\\
    \frac{4}{4\left(m_{1}+m_{2}\right)-3 m_{2} \cos ^{2} x_{4}} \\
    \frac{-6\cos x_{4}}{4 l\left(m_{1}+m_{2}\right)-3 m_{2} l \cos ^{2} x_{4}}\\
    0
    \end{array}\right]u,
\end{align*}}where $[x_{1},x_{2},x_{3},x_{4}]=[x,\dot{x},\dot{\theta},\theta]$, $x$ is the position of the cart along the track, $\theta$ is the pendulum angle measured anti-clockwise from hanging down. The control input $u$ is the force applied to the cart horizontally. Pole length is $l=0.6$ m; mass of cart and pole are $M=m=0.5~\mathrm{kg}$; friction parameter is $b=0.1$ N/m/s and gravity is $g=9.82~\mathrm{m/s^{2}}$.
\subsubsection{Pendubot} The dynamics of this pendubot is detailed in \cite{danthesis}, and the model parameters were obtained via system identification. The system is given by
\begin{equation*}
\begin{array}{l}
{\left[\begin{array}{l}
\ddot{q}_{1} \\
\ddot{q}_{2}
\end{array}\right]=D(q)^{-1} \tau-D(q)^{-1} C(q, \dot{q}) \dot{q}-D(q)^{-1} g(q)},\\
\end{array}
\end{equation*}
where $q_{1}$ is the joint angle measured anti-clockwise between the positive $x$-axis direction and Link~1, and $q_{2}$ is the joint angle measured anti-clockwise between Link~2 and the direction of the central axis of Link~1 towards Link~2. The three matrices are defined as follows:
{\setlength\arraycolsep{3pt}
\begin{equation*}
\begin{aligned}
D(q) &\!=\!\left[\begin{array}{cc}
\theta_{1}+\theta_{2}+2 \theta_{3} \cos q_{2} & \theta_{2}+\theta_{3} \cos q_{2} \\
\theta_{2}+\theta_{3} \cos q_{2} & \theta_{2}
\end{array}\right], \\
C(q, \dot{q}) &\!=\!\left[\begin{array}{cc}
-\theta_{3} \sin \left(q_{2}\right) \dot{q}_{2} & -\theta_{3} \sin \left(q_{2}\right) \dot{q}_{2}-\theta_{3} \sin \left(q_{2}\right) \dot{q}_{1} \\
\theta_{3} \sin \left(q_{2}\right) \dot{q}_{1} & 0
\end{array}\right], \\
g(q) &\!=\!\left[\begin{array}{c}
\theta_{4} g \cos q_{1}+\theta_{5} g \cos \left(q_{1}+q_{2}\right) \\
\theta_{5} g \cos \left(q_{1}+q_{2}\right)
\end{array}\right],
\end{aligned}
\end{equation*}
where $[\theta_{1}, \theta_{2}, \theta_{3}, \theta_{4}, \theta_{5}] = [0.00348, 0.00120, 0.00107, 0.93342,0.28043]$. Furthermore, $\theta_{1}$, $\theta_{2}$ and $\theta_{3}$ are unitless, and the unit of $\theta_{4}$ and $\theta_{5}$ is m.}

\subsubsection{Quadrotor} The dynamics is taken from \cite{bouabdallah2004design-quadrotor}, which use Euler angles. The system is given by
\begin{equation*}
\begin{array}{l}
\ddot{x}=(\cos \phi \cos \theta \cos \psi+\sin \phi \sin \psi) \frac{f_{z}}{m} \\
\ddot{y}=(\cos \phi \sin \theta \sin \psi-\sin \phi \cos \psi) \frac{f_{z}}{m} \\
\ddot{z}=\cos \phi \cos \theta \frac{f_{z}}{m}-g \\
\ddot{\phi}=\dot{\theta} \dot{\psi}\left(\frac{I_{y}-I_{z}}{I_{x}}\right)+\frac{\tau_{\phi}}{I_{x}}  \\
\left.\ddot{\theta}=\dot{\phi} \dot{\psi} (\frac{I_{z}-I_{x}}{I_{y}}\right)+\frac{\tau_{\theta}}{I_{y}}  \\
\ddot{\psi}=\dot{\phi} \dot{\theta}\left(\frac{I_{x}-I_{y}}{I_{z}}\right)+\frac{\tau_{\rho}}{I_{z}}, 
\end{array}
\end{equation*}
where $x, y, z$ represent the position of center of mass in the inertial frame, $\phi, \theta, \psi$ are the roll, pitch and yaw angles. The control inputs are the total thrust $f_{z}$ and torques $\tau_{\phi},\tau_{\theta},\tau_{\psi}$ around the three axes. The total mass of quadrotor is $m = 4.34~\mathrm{kg}$; the moments of inertia around three axes are $I_{x} = 0.082~\mathrm{\mathrm{kg} m^{2}}$, $I_{y} = 0.0845~\mathrm{\mathrm{kg} m^{2}}$ and $I_{z} = 0.1377~\mathrm{\mathrm{kg} m^{2}}$.

\bibliographystyle{ieeetr}
\bibliography{bib/refs}

\end{document}